\documentclass{article}

\usepackage[preprint]{neurips_2026}

\usepackage[utf8]{inputenc} 
\usepackage[T1]{fontenc}    
\usepackage{hyperref}       
\usepackage{url}            
\usepackage{booktabs}       
\usepackage{amsfonts}       
\usepackage{nicefrac}       
\usepackage{microtype}      

\usepackage{graphicx}
\usepackage{subfigure}
\usepackage{multirow} 
\usepackage{makecell}
\usepackage{bbding}
\usepackage{pifont}
\usepackage[table]{xcolor}
\usepackage{stfloats}
\usepackage{mdframed}
\usepackage{mathtools}

\usepackage{amsmath}
\usepackage{amsthm}
\usepackage{algorithm}
\usepackage{algorithmic}

\newtheorem{theorem}{Theorem}

\title{DiPO: Disentangled Perplexity Policy Optimization for Fine-grained Exploration-Exploitation Trade-Off}

\author{%
 \makecell{Xiaofan Li$^{1,2}$, Ming Yang$^2$\thanks{Equan Contribution}, Zhiyuan Ma$^2$, Shichao Ma$^2$, Jintao Du$^2$, Yu Cheng$^2$, \\ Weiqiang Wang$^2$, Zhizhong Zhang$^1$, Xin Tan$^1$, Yanyun Qu$^4$, Lizhuang Ma$^1$, Yuan Xie$^{1,3}$\thanks{Corresponding Author}}  \\
  $^1$East China Normal University, $^2$Ant Group \\
  $^3$Shanghai Innovation Institute, $^4$Xiamen University\\
  \texttt{lxfunzi@stu.ecnu.edu.cn},
  \texttt{\{lzma, yxie\}@cs.ecnu.edu.cn}
}

\begin{document}

\maketitle

\begin{abstract}
  Reinforcement Learning with Verifiable Rewards (RLVR) has catalyzed significant advances in the reasoning capabilities of Large Language Models (LLMs). However, effectively managing the exploration and exploitation trade-off remains a critical challenge. In this paper, we fully analyze the exploration and exploitation dilemma of extremely hard and easy samples during the training and propose a new fine-grained trade-off mechanism. Concretely, we introduce a perplexity space disentangling strategy that divides the sample space into distinct exploration (high perplexity) and exploitation (low perplexity) subspaces, thereby mining fine-grained samples requiring exploration-exploitation trade-off. Subsequently, we propose a bidirectional reward allocation mechanism with a minimum impact on verification rewards to implement perplexity-guided exploration and exploitation, enabling more stable policy optimization. Finally, we have evaluated our method on two mainstream tasks: mathematical reasoning and function calling, and experimental results demonstrate the superiority of the proposed method, confirming its effectiveness in enhancing LLM performance by fine-grained exploration-exploitation trade-off.

\end{abstract}

\section{Introduction}


Exploration-Exploitation Trade-Off (EETO) \cite{multiarmed} represents a core challenge in RL for LLM post-training, where error samples stand to benefit from further exploration, while correct samples are suitable for exploitation. In this paper, we propose the two dilemmas of EETO in GRPO-based methods. \textbf{First dilemma: the degradation of the advantage of extreme samples.} As illustrated in Figure \ref{fig: intro} (a), a large proportions of sample groups fall into either the hard group (with uniform zero rewards) or the easy group (with uniform one rewards) during training. These extreme groups yield a zero advantage, leading to the lack of training gradients and thus exposes the GRPO paradigm to an dilemma of ETTO, where the hard and easy groups lack signals for exploration and exploitation, respectively. \textbf{Second dilemma: ineffective exploration and exploitation.} As shown in Figure \ref{fig: intro} (b), the distribution of perplexity (PPL) reflects the exploratory (high PPL) or exploitative (low PPL) tendency of samples, where some error samples exhibit an exploitative tendency, while some correct samples show an exploratory one. This ineffective exploration and exploitation severely decrease the stability and plasticity of RL.

\begin{figure}[t]
  \centering
   \includegraphics[width=1.0\linewidth]{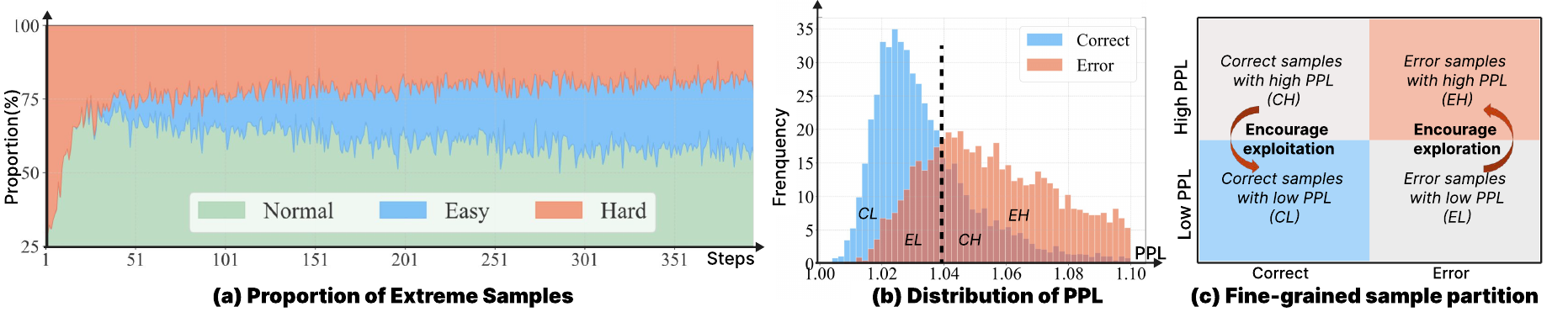}
    
   \caption{\textbf{(a)} The proportion of Easy/Normal/Hard groups in each step during the DAPO training. \textbf{(b)} The PPL distribution of correct and error samples in the validation set at 300th steps of DAPO training. \textbf{(c)} Illustration of four samples after PSD fine-grained partitioning.} 
   \label{fig: intro}
\end{figure}

Recent approaches have investigated the EETO by introducing PPL reward shaping. DACE \cite{DACE} encourages all samples either to explore or to exploit based on intra-group accuracy. However, this coarse-grained partitioning would hinder policy optimization during early training stages when most samples belong to the hard group. Meanwhile, CDE \cite{CDE} designs multiple weighting mechanisms to employ exploration rewards at appropriate times, yet it does not explicitly address exploitation and introduces more hyper-parameters. Moreover, both CDE and DACE directly use perplexity as a reward bias, which may introduce additional uncertainty. In this paper, we propose a novel method \textbf{Di}sentangled Perplexity \textbf{P}olicy \textbf{O}ptimization (\textbf{DiPO}) to solve the two dilemmas and achieve a fine-grained EETO.

To enable fine-grained sample mining for exploratory and exploitative purposes, we propose the Perplexity Space Disentangling (PSD) strategy, that integrates statistical probability of correctness (verification reward) and PPL. In Figure \ref{fig: intro} (b), error samples typically exhibit higher PPL, whereas correct samples generally maintain lower PPL. PSD infers the optimal threshold $\tau*$ by advantage judgment and minimizing classification errors, thus capturing the inherent correlation between PPL and sample correctness. Subsequently, as shown in Figure \ref{fig: intro} (c), the entire sample space is partitioned into four fine-grained quadrants: correct samples with high PPL (CH), correct samples with low PPL (CL), error samples with high PPL (EH), and error samples with low PPL (EL). While the CL and EH samples naturally align with effective exploitation and exploration; The critical of EETO lies in encouraging exploitation for CH samples and exploration for EL samples.

In addition, the uncertainty of PPL distribution results in a significant discrepancy between the PPL reward and the verification reward. To achieve stable RL training for EETO, we abandon the direct adoption of PPL for reward shaping and proposes a Bidirectional Reward Reallocation (BRR) mechanism. Specifically, on the one hand, to avoid interfering with the policy optimization guided by the verification reward, BRR is only performed on the zero-gradient easy and hard groups. On the other hand, because the verification reward variance of the easy and hard groups is zero, we introduce a maximum-PPL reward reallocation strategy, where the rewards corresponding to the maximum PPL samples in the easy and hard groups are set to 0 and 1, respectively, ensuring the variance of the reallocated reward distribution is close to zero. The proposed BRR incorporates the PPL signal while minimizing the perturbation to the original reward distribution, thus realizing more stable training.

To summarize, the main contributions of this paper are:
\begin{itemize}
  \item We propose a perplexity space disentanglement strategy that partitions the PPL space based on perplexity and correctness distributions, enabling a fine-grained EETO.
  \item We design a bidirectional reward reallocation mechanism that stabilizes training by reallocating rewards with minimal perturbation to the verification reward distribution.”
  \item For mathematical reasoning and function calling tasks, our method achieves superior results confirming its effectiveness in reasoning enhancement for LLMs.

\end{itemize}

\begin{figure*}[t]
  \centering
   \includegraphics[width=1.0\linewidth]{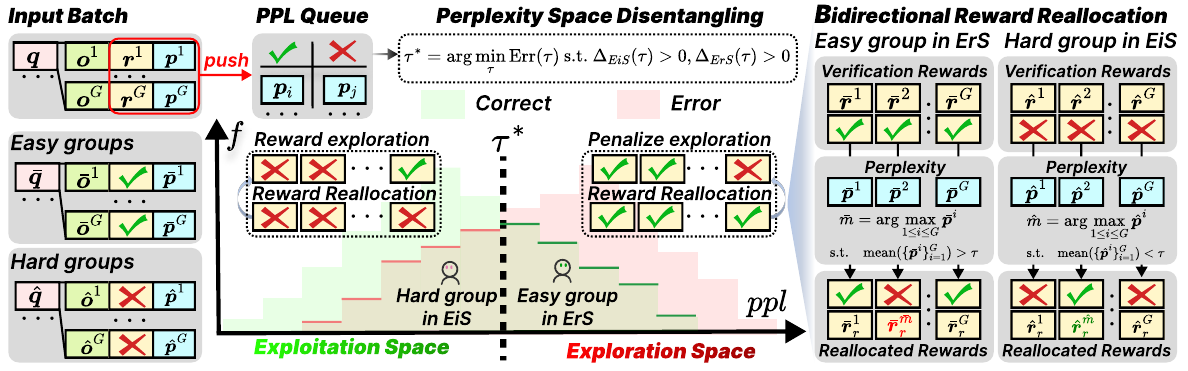}
    
   \caption{Illustration of DiPO, consisting of three modules: PPL Queue, Perplexity Space Disentangling (PSD), and Bidirectional Reward Reallocation (BRR). Specifically, the PPL Queue caches PPL items; PSD is used for fine-grained sample partition; and then BRR performs reward allocation.}
   \label{method}
\end{figure*}

\section{Preliminaries and Definitions}
\subsection{RLVR Algorithms}
\textbf{Group Relative Policy Optimization (GRPO)} introduces a concise advantage calculation method. For each query $\boldsymbol{q}$ and its ground-truth answer $\boldsymbol{a}$, given the verifiable reward calculation function:
\begin{equation}
    \begin{aligned}
    \mathcal{R}(\boldsymbol{o}^i, \boldsymbol{a}) = 1 \quad \texttt{if is\_equivalent}(\boldsymbol{a}, \boldsymbol{o}^i) \ \texttt{else} \quad 0,
    \end{aligned}
\label{eq: reward function}
\end{equation}
where $\{\boldsymbol{o}^{i}\}_{i=1}^{G}$ is generated by rollout policy $\pi_{\boldsymbol{\theta_{\text{old}}}}(\boldsymbol{o}|\boldsymbol{q})$ through $G$-sampling. The estimated advantage $\hat{\boldsymbol{A}}^i_{t}$ under $\mathcal{R}$ is then computed as:
\begin{equation}
    \begin{aligned}
    \hat{\boldsymbol{A}}^i_{t}(\mathcal{R}) = \frac{\mathcal{R}(\boldsymbol{o}^i,\boldsymbol{a}) - \text{mean}(\{\mathcal{R}(\boldsymbol{o}^i,\boldsymbol{a})\}_{i=1}^{G})}{\text{std}(\{\mathcal{R}(\boldsymbol{o}^i,\boldsymbol{a})\}_{i=1}^{N})}.
    \end{aligned}
\label{eq: grpo advantage}
\end{equation}
When G-sampling rewards are all 0 or 1, all advantages come to 0, indicating the sampled group consists solely of {hard groups} and {easy groups}.



\textbf{Dynamic Sampling Policy Optimization (DAPO)} introduces new improvements based on GRPO The maximum optimization objective under $\mathcal{R}$ is as follows:
\begin{equation}
    \resizebox{1\hsize}{!}{$
    \begin{aligned}
         \mathcal{J}_{\text{DAPO}}(\theta, \mathcal{R}) = \mathbb{E}_{(\boldsymbol{q}, \boldsymbol{a}) \sim \mathcal{D}, \{\boldsymbol{o}^{i}\}^G_{i=1} \sim \pi_{{\boldsymbol{\theta}}_{\text{old}}}(\boldsymbol{o|\boldsymbol{q}})} 
         \frac{1}{\sum_{i=1}^G |\boldsymbol{o}^{i}|}  \sum_{i=1}^{G} \sum_{t=1}^{|\boldsymbol{o}^{i}|}
        \Bigg[
        \min \Big( 
         \boldsymbol{s}_{t}^i \hat{\boldsymbol{A}}_{t}^i(\mathcal{R}),\, \text{clip}\big(\boldsymbol{s}_{t}^i, 1 - \epsilon_{\text{low}}, 1 + \epsilon_{\text{high}} \big) \hat{\boldsymbol{A}}_{t}^i(\mathcal{R}) 
        \Big) \Bigg], \\
        \text{s.t.} \ \ \ 0 < | \{\mathcal{R}(\boldsymbol{o}^i, \boldsymbol{a}) | \mathcal{R}(\boldsymbol{o}^{i}, \boldsymbol{a}) = 1 )\} | < G
    \end{aligned}
    $}
    \label{eq: dapo}
\end{equation}
where $\quad \boldsymbol{s}_t^i = \frac{\pi_{\boldsymbol{\theta}}(\boldsymbol{o}_t^i | \boldsymbol{q}, \boldsymbol{o}_{<t}^i)}{\pi_{\boldsymbol{\theta}_{\text{old}}}(\boldsymbol{o}_t^i | \boldsymbol{q}, \boldsymbol{o}_{<t}^i)}$ is the ratio of importance sampling. In this paper, we use advanced DAPO as our baseline without {overlong reward shaping}, as in some cases, overlong reward shaping would damage the model's performance. 

\subsection{Perplexity}
Perplexity (shorted as PPL) in Large Language Models is a metric that quantifies how confidently a probability model predicts a sample. Mathematically, given a model $\pi_{\boldsymbol{\theta}}$, a question $\boldsymbol{q}$, and a response $\boldsymbol{o}^i = (\boldsymbol{o}^i_1, \boldsymbol{o}^i_2, ..., \boldsymbol{o}^i_T)$ generated by $\pi_{\boldsymbol{\theta}}$, the perplexity is given by:
\begin{equation}
    \boldsymbol{p}^i = \exp\left(-\frac{1}{T}\sum_{t=1}^{T} \log \pi_{\theta}(\boldsymbol{o}^i_t|\boldsymbol{q}, \boldsymbol{o}_{<t})\right)
\label{eq: ppl}
\end{equation}
For a response, a low PPL indicates that the trajectory tends to \emph{exploitation}, while higher PPL indicates that the trajectory tends to \emph{exploration} \cite{CDE, pplcrl}.

\section{Methodology}
\subsection{Overview}

As illustrated in Figure \ref{method}, the proposed DiPO mainly contains two modules: {1) Perplexity Space Disentangling (PSD)}; {2) Bidirectional reward reallocation (BRR)}. PSD calculates conditional probability distribution based on the PPL Queue cached with PPL-reward pairs, and then determines the optimal threshold by the advantage judgment and minimizing classification errors. We term higher and lower PPL region exploration space (ErS) and exploitation space (EiS), respectively. Through PPL disentanglement, the ineffective samples for exploration and exploitation, i.e., hard groups in EiS and Easy groups in ErS are selected for policy optimization. Subsequently, BRR is implemented for these samples; for the hard group in EiS, the reward for maximum PPL sample is set to 1, thus encouraging exploration; the easy group in ErS, the reward for maximum PPL sample is set to zero, thus penalizing exploration. BRR incorporates the PPL signal and minimizes perturbations to the verification reward, enabling more stable training.

\subsection{Perplexity Space Disentangling}
PSD calculates the optimal threshold $\tau^*$ for dividing EiS and ErS by introducing the conditional probability distribution $\mathrm{Pr}(R|P)$ where $P$ and $R$ represent PPL and reward. For ETTO, error samples stand to benefit from further exploration, while correct samples are suitable for exploitation. Thus the optimal EiS and ErS should have the highest relevance to sample correctness, where samples falls in EiS are considered more likely to be correct, whereas those in ErS are generally deemed error. 

\textbf{Online Statistical Estimation.} Since the PPL distribution induced by the policy changes dynamically during training, direct estimation of conditional probabilities may suffer from an unfavorable variance-bias trade-off. To stabilize the estimation, we maintain a {PPL queue} $\mathcal{Q}$ that caches samples from the most recent two batches:
\(\mathcal{Q}=\{(\boldsymbol{p}_i,\boldsymbol{r}_i)\}_{i=1}^{M},\)
where $\boldsymbol{p}_i\in [1, +\infty)$ denotes the perplexity (PPL) of the $i$-th sample and $\boldsymbol{r}_i\in\{0,1\}$ denotes its corresponding validation reward.

For a given threshold $\tau$, we estimate the empirical conditional probabilities of the reward given whether the PPL is below or above $\tau$ as
\begin{equation}
\label{eq:conditional_probs}
\resizebox{.8\hsize}{!}{$
\begin{aligned}
\widehat{\Pr}(R=1 \mid P<\tau) &= \frac{\sum_{i=1}^{M} \mathbb{I}(p_i<\tau,\, r_i=1)}{\sum_{i=1}^{M} \mathbb{I}(p_i<\tau)}, \ \
\widehat{\Pr}(R=0 \mid P<\tau) &= \frac{\sum_{i=1}^{M} \mathbb{I}(p_i<\tau,\, r_i=0)}{\sum_{i=1}^{M} \mathbb{I}(p_i<\tau)}, \\
\widehat{\Pr}(R=1 \mid P>\tau) &= \frac{\sum_{i=1}^{M} \mathbb{I}(p_i>\tau,\, r_i=1)}{\sum_{i=1}^{M} \mathbb{I}(p_i>\tau)}, \ \
\widehat{\Pr}(R=0 \mid P>\tau) &= \frac{\sum_{i=1}^{M} \mathbb{I}(p_i>\tau,\, r_i=0)}{\sum_{i=1}^{M} \mathbb{I}(p_i>\tau)} .
\end{aligned}
$}
\end{equation}
where $\mathbb{I}(\cdot)$ denotes the indicator function.

To quantify the statistical uncertainty of each estimate, we compute $95\%$ confidence intervals using the Normal (Wald) approximation \cite{wald}. For a generic empirical probability estimate $\hat{p}=m/n$, where $m$ denotes the number of successes among $n$ samples, the standard error is
\(\mathrm{SE}(\hat{p})=
\sqrt{\hat{p}(1-\hat{p})/n},
\)
and the corresponding $95\%$ confidence interval is
\begin{equation}
\label{eq:wald_ci}
\left[\,\hat{p}-1.96\cdot \mathrm{SE}(\hat{p}),\;\hat{p}+1.96\cdot \mathrm{SE}(\hat{p})\,\right].
\end{equation}
These estimated probabilities represent the sample likelihood of being correct or error in the \textit{Exploitation Space} (EiS, $P<\tau$) and \textit{Exploration Space} (ErS, $P>\tau$) demarcated by the threshold $\tau$.

\textbf{Advantage Judgment.}
During the early RL stage of pre-trained models, the correlation between PPL and correctness may not be positively correlated (see Appendix \ref{sec: app ppl}), where directly activating PSD would hinder optimization. To this end, we introduce a advantage judgment mechanism based on the results of online statistical estimation.

Specifically, for a given threshold $\tau$, we define two advantage functions that characterize the separation between correct samples and error samples in divided PPL regions:
\begin{equation}
\label{eq:correlation_gap}
\begin{aligned}
\Delta_{\mathrm{EiS}}(\tau) 
&= \Pr(R=1 \mid P<\tau) - \Pr(R=0 \mid P<\tau), \\
\Delta_{\mathrm{ErS}}(\tau) 
&= \Pr(R=0 \mid P>\tau) - \Pr(R=1 \mid P>\tau),
\end{aligned}
\end{equation}
where $\Delta_{\mathrm{EiS}}(\tau)$ measures the \emph{correctness advantage} in the EiS, and $\Delta_{\mathrm{ErS}}(\tau)$ measures the \emph{error advantage} in the ErS.

To reduce the impact of sampling randomness and ensure robust judgment, we conservatively evaluate these gaps using the boundary values of the $95\%$ confidence intervals derived in the online estimation stage ({Equation (\ref{eq:wald_ci})). Concretely, the two judgment functions are as follows:
\begin{equation}
\label{eq:ci_gap}
\Delta_{\mathrm{EiS}}(\tau) = L_1 - U_2\quad \text{and} \quad
\Delta_{\mathrm{ErS}}(\tau) = L_4 - U_3,
\end{equation}
where $[L_1, U_1]$ and $[L_2, U_2]$ denote the $95\%$ confidence intervals of $\Pr(R=1 \mid P<\tau)$ and $\Pr(R=0 \mid P<\tau)$, respectively, $[L_3, U_3]$, $[L_4, U_4]$ correspond to $\Pr(R=1 \mid P>\tau)$ and $\Pr(R=0 \mid P>\tau)$. Considering the PPL space to be meaningfully correlated with correctness at threshold $\tau$ only when both conditions
\(\Delta_{\mathrm{EiS}}(\tau) > 0\) and \(\Delta_{\mathrm{ErS}}(\tau) > 0\) hold.

\textbf{Minimizing Classification Errors.} In the middle and later RL stages, a wide range of thresholds may satisfy Advantage Judgment, necessitating the further selection of an optimal value. We formalize a classification task using PPL as the criterion: responses with a PPL below $\tau$ are classified as correct, while those above are deemed error. By minimizing classification errors ($\text{Err}(\tau)$ in Figure \ref{method}) to find the optimal threshold:
\begin{equation}
\label{eq: opt_t}
\begin{aligned}
& \tau* = \arg\min_{\tau} \frac{1}{|\mathcal{Q}|} \sum_{\scriptsize{(\boldsymbol{r}_i, \boldsymbol{p}_i) \in \mathcal{Q}}} \left| \boldsymbol{r}_i - \mathbb{I}(\boldsymbol{p}_i < \tau) \right| \\
& \text{s.t.} \ \ \Delta_{EiS}(\tau)>0, \ \ \Delta_{ErS}(\tau)>0
\end{aligned}
\end{equation}
where $\mathbb{I}(\cdot)$ is the indicator function. Through advantage judgment and minimizing classification errors, the disentangled PPL spaces are highly correlated with correctness. Consequently, the hard groups in EiS and the easy groups in ErS would be the important samples for achieving the EETO. Algorithm details are illustrated in Algorithm \ref{alg: psd}.

\subsection{Bidirectional Reward Reallocation}
\label{sec: brr}
To implement a stable exploration-exploitation optimization and minimize the impact to verification rewards, BRR does not directly employ PPL as the basis for reward shaping, but introduces the maximum-PPL reward reallocation strategy. BRR thus aims to drive updates toward the exploratory (high-entropy) direction for the hard groups in EiS, and toward the exploitative (low-entropy) direction for the easy groups in ErS..
 
\begin{theorem}[Entropy Increase with Maximum-PPL Reward]
\label{thm:grpo-entropy-increase}
Let $\pi_t$ be the policy at step $t$. Given a query $\boldsymbol{q}$ and a group of outputs $\{\boldsymbol{o}^i\} \sim \pi_t(\cdot \mid \boldsymbol{q})$, if the output with the maximum PPL among $\{\boldsymbol{o}^i\}$ is assigned a reward, then the average entropy of the updated policy $\pi_{t+1}(\cdot \mid \boldsymbol{q})$ increases.
\end{theorem}

\textbf{Reward Reallocation for Hard Groups.} As shown in Figure \ref{method}, we define a hard group as $\hat{\boldsymbol{q}}$, $\{\hat{\boldsymbol{o}}^i\}$, $\{\hat{\boldsymbol{r}}^i\}$, $\{\hat{\boldsymbol{p}}^i\}$, here $\{ \hat{\boldsymbol{r}}^i \}=\{ 0 \}$. If ${mean}(\{\hat{\boldsymbol{p}}^i\}) < \tau*$ means that this hard group is located in EiS, the goal of the reward reallocation is to update it towards the high entropy direction. Concretely, the index $\hat{m}$ corresponding to the maximum PPL in $\{\hat{\boldsymbol{o}}^i\}$ is identified: $\hat{m} = \arg\max_{1 \leq i \leq G} \hat{\boldsymbol{p}}^i$. Subsequently, the reallocated reward $\{\hat{\boldsymbol{r}}^i_r\}$ is constructed by setting $\hat{\boldsymbol{r}}^m_r=1$ while preserving all other original values, such that $\hat{\boldsymbol{r}}^i_r=\hat{\boldsymbol{r}}^i$ for all $i \not= \hat{m}$. Noted that if $\tau*$ does not exist, or ${mean}(\{\hat{\boldsymbol{p}}^i\}) > \tau*$, reallocated reward remains unchanged, \text{i.e.}, $\{\hat{\boldsymbol{r}}^i_r\}=\{\hat{\boldsymbol{r}}^i\}=\{ 0 \}$. 

\begin{theorem}[Entropy decrease with Maximum-PPL penalty]
\label{thm:grpo-entropy-increase}
Let $\pi_t$ be the policy at step $t$. Given a query $\boldsymbol{q}$ and a group of outputs $\{\boldsymbol{o}^i\} \sim \pi_t(\cdot \mid \boldsymbol{q})$, if the output with the maximum PPL among $\{\boldsymbol{o}^i\}$ is assigned a penalty, then the average entropy of the updated policy $\pi_{t+1}(\cdot \mid \boldsymbol{q})$ decreases.
\end{theorem}

\textbf{Reward Reallocation for Easy Groups.} As shown in Figure \ref{method}, reward reallocation for the easy group is the opposite of that for the hard group. A easy group is defined as $\bar{\boldsymbol{q}}$, $\{\bar{\boldsymbol{o}}^i\}$, $\{\bar{\boldsymbol{r}}^i\}$, $\{\bar{\boldsymbol{p}}^i\}$, here $\{ \bar{\boldsymbol{r}}^i \}=\{ 1 \}$. If ${mean}(\{\bar{\boldsymbol{p}}^i\}) > \tau*$ means that this easy group is located in ErS, the goal of the reward reallocation is to update it towards the low entropy direction. Concretely, the index $\bar{m}$ corresponding to the maximum PPL in $\{\bar{\boldsymbol{o}}^i\}$ is identified: $\bar{m} = \arg\max_{1 \leq i \leq G} \bar{\boldsymbol{p}}^i$. Subsequently, the reallocated reward $\{\bar{\boldsymbol{r}}^i_r\}$ is constructed by setting $\bar{\boldsymbol{r}}^m_r=0$ while preserving all other original values, such that $\bar{\boldsymbol{r}}^i_r=\bar{\boldsymbol{r}}^i$ for all $i \not= \bar{m}$. Noted that if $\tau*$ does not exist, or ${mean}(\{\bar{\boldsymbol{p}}^i\}) < \tau*$, reallocated reward remains unchanged, \text{i.e.}, $\{\bar{\boldsymbol{r}}^i_r\}=\{\bar{\boldsymbol{r}}^i\}=\{ 1 \}$.

Finally, we abstracted the peocess of BRR into a formalized reward function $\mathcal{R}_r$, defined as follows:
\begin{equation}
\mathcal{R}_r(\{\boldsymbol{r}^i\})=\begin{cases}
\{\hat{\boldsymbol{r}}^i_r\} & \text{if } \{\boldsymbol{r}^i\}=\{0\}, \\
\{\bar{\boldsymbol{r}}^i_r\} & \text{if } \{\boldsymbol{r}^i\}=\{1\}, \\
\{0\} & \text{other},\end{cases}
\end{equation}
where BRR only performs reward reallocation for the easy groups and hard groups, and sets all rewards of the normal groups to zero, thus ensuring that the reallocated rewards (denoted as $\mathcal{R}_r$) and verification rewards (denoted as $\mathcal{R}$) are orthogonal. Algorithm details are illustrated in Algorithm \ref{alg: brr}.

\textbf{Policy Optimization.} $\mathcal{R}$ and $\mathcal{R}_r$ are orthogonal, which means that the samples with gradients under the two rewards are orthogonal. Thus, we use a hyper-parameter $\alpha$ to control the loss weight of $\mathcal{R}_r$. The maximum optimization objective is as follows:
\begin{equation}
\mathcal{J}_{DiPO}(\theta) = \mathcal{J}_{DAPO}(\theta, \mathcal{R}) + \alpha \times \mathcal{J}_{DAPO}(\theta, \mathcal{R}_r),
\end{equation}
where $\mathcal{J}_{DAPO}$ is the DAPO optimization objective in Eq. (\ref{eq: dapo}).

\section{Experiment}
\subsection{Experiment Setup}
To comprehensively evaluate the effectiveness of our proposed DiPO, we conduct experiments on two downstream tasks of LLMs: mathematical reasoning and function calling. Detailed configurations are provided below.

\begin{table*}[]
\centering
\caption{Comparison of mathematical reasoning in ACC/mean@8 on 6 mathematics benchmarks. The best and second-best results are respectively marked in \textbf{bold} and underlined. }
\scalebox{0.87}{
\setlength{\tabcolsep}{4mm}
\begin{tabular}{l|c|c|c|c|c|c|c}
\toprule
\textbf{Method} & {\textbf{AIME24}} & {\textbf{AIME25}} & {\textbf{MATH}} & {\textbf{AMC}} & {\textbf{OLY}} & {\textbf{MIN}} & {\textbf{AVG}}       \\ \midrule
\multicolumn{8}{c}{\cellcolor[RGB]{239, 239, 239}Qwen3-4B-Base}   \\ \midrule
Base model     & 9.58     & 3.75     & 48.48  & 31.48 & 25.52 & 24.82 & \cellcolor[RGB]{239, 243, 248}23.94          \\
GRPO           & \underline{26.67} & 23.33    & 85.83  & 60.24 & 53.06 & 44.39 & \cellcolor[RGB]{239, 243, 248}48.92          \\
DAPO           & 26.25    & 23.75    & 86.43  & 61.90 & 53.88 & 44.34 & \cellcolor[RGB]{239, 243, 248}49.43          \\
DAPO w/ EL     & \underline{26.67} & \textbf{24.58}    & \underline{86.78} & \underline{62.95}    & \textbf{54.53} & \underline{44.53}    & \cellcolor[RGB]{239, 243, 248}\underline{50.01}    \\
CDE            & \underline{26.67} & \underline{24.17} & 85.93  & 62.35 & 52.25 & 43.11 & \cellcolor[RGB]{239, 243, 248}49.08          \\ \midrule
\textbf{DiPO (ours)}          & \textbf{29.17}    & \textbf{24.58}    & \textbf{87.00}  & \textbf{63.70} & \underline{54.09}    & \textbf{44.76} & \cellcolor[RGB]{239, 243, 248}\textbf{50.55} \\ \midrule
\multicolumn{8}{c}{\cellcolor[RGB]{239, 239, 239}Qwen3-8B-Base}   \\ \midrule
Base model     & 8.33     & 9.17     & 66.78  & 39.46 & 33.30 & 66.78 & \cellcolor[RGB]{239, 243, 248}37.30          \\
GRPO           & 31.67    & 24.58    & 89.08  & 69.28 & 56.20 & \textbf{48.62} & \cellcolor[RGB]{239, 243, 248}53.24          \\
DAPO           & 30.08    & 25.83    & 89.43  & 69.12 & 56.90 & \underline{48.02} & \cellcolor[RGB]{239, 243, 248}53.23          \\
DAPO w/ EL     & \underline{33.75} & 25.42    & \textbf{89.58}  & \underline{69.87}    & \underline{57.21}    & {47.56}    & \cellcolor[RGB]{239, 243, 248}\underline{53.90}    \\
CDE            & 31.67    & \underline{26.25} & 89.35  & 68.07 & 57.11 & {47.75} & \cellcolor[RGB]{239, 243, 248}53.37          \\ \midrule
\textbf{DiPO (ours)}          & \textbf{35.00}    & \textbf{27.50}    & \underline{89.55} & \textbf{71.23} & \textbf{57.73} & 47.75 & \cellcolor[RGB]{239, 243, 248}\textbf{54.79} \\ \midrule
\multicolumn{8}{c}{\cellcolor[RGB]{239, 239, 239}Qwen2.5-7B}   \\ \midrule
Base model     & 7.08     & 2.08     & 41.53 & 22.74 & 19.28 & 17.19 & \cellcolor[RGB]{239, 243, 248}18.32          \\
GRPO           & \underline{20.42} & 15.42    & 79.15  & 58.43 & 42.42 & 36.95 & \cellcolor[RGB]{239, 243, 248}42.13          \\
DAPO           & \underline{20.42} & \underline{16.67} & 79.08  & \underline{59.94}    & 42.70 & 37.55 & \cellcolor[RGB]{239, 243, 248}\underline{42.73}    \\
DAPO w/ EL     & 20.00    & 14.58    & \underline{79.85} & 58.73 & \underline{43.05}    & \textbf{39.65} & \cellcolor[RGB]{239, 243, 248}42.64          \\
CDE            & 20.00    & 15.00    & 79.00  & 55.87 & 42.94 & 35.94 & \cellcolor[RGB]{239, 243, 248}41.46          \\ \midrule
\textbf{DiPO (ours)}          & \textbf{22.92}    & \textbf{16.67}    & \textbf{80.35}  & \textbf{60.09} & \textbf{43.72} & \underline{37.59}    & \cellcolor[RGB]{239, 243, 248}\textbf{43.56} \\ \bottomrule
\end{tabular}
}
\label{tab: math}
\end{table*}

\textbf{Mathematical Reasoning.} We evaluate DiPO on the mathematical reasoning task using the DAPO-17K \cite{DAPO} dataset for training. Model performance is assessed on five challenging benchmarks: AIME24, AIME25, AMC23 (denoted as AMC), MATH500 \cite{math500} (denoted as MATH) the OE\_TO\_mat\_en\_COMP subset from OlympiadBench \cite{oly}(denoted as OLY), and Minerva (denoted as MIN). These evaluations are designed to comprehensively examine the method's effectiveness across diverse problem difficulties and formats.

We perform a comparative analysis against several strong baselines: 1. GRPO \cite{GRPO}, 2. DAPO \cite{DAPO}, 3. DAPO enhanced with Entropy Loss (DAPO w/ EL) \cite{el}, and 4. CDE \cite{CDE} with PPL reward shaping. All mathematical reasoning experiments are implemented using the VERL \cite{verl} framework. To ensure a rigorous comparison, we employ three different pretrained base models: Qwen3-4B-Base, Qwen3-8B-Base \cite{qwen3}, and Qwen2.5-7B \cite{qwen2.5}. All models utilize consistent training configurations and prompts, as detailed in {Appendix} \ref{sec: app math}.

\textbf{Function Calling.} We establish our experimental setup using the publicly available ToolRL \cite{toolrl} framework as the primary baseline. Furthermore, we implement ToolRL augmented with DAPO (ToolRL+DAPO) for comparison. Since ToolRL does not utilize a binary verification reward, certain modifications were made when implementing TooRL+DiPO.

We conduct comprehensive testing on the BFCLv3 benchmark \cite{bfcl}, which provides a diverse set of function calling scenarios for thorough evaluation. Consistent with the TooRL setup, we use the same model architectures (Qwen2.5-3B-Instruct, Qwen2.5-7B-Instruct) and maintain identical training configurations and datasets across all compared methods, details in see Appendix \ref{sec: app tool}.

\begin{table*}[t!]
\centering
\caption{Comparison of function calling in acc on BFCLv3 benchmark. The best and second-best results are respectively marked in \textbf{bold} and underlined.}
\scalebox{.87}{
\setlength{\tabcolsep}{3mm}
\begin{tabular}{l|c|c|c|c|c|c}
\toprule
\textbf{Method}           & \textbf{Non-Live Acc} & \textbf{Live Acc} & \makecell{\textbf{Multi Turn} \\ \textbf{Acc}} & \makecell{\textbf{Relevance} \\ \textbf{Detection}} & \makecell{\textbf{Irrelevance} \\ \textbf{Detection}} & \textbf{Overall} \\ \midrule
\multicolumn{7}{c}{\cellcolor[RGB]{239, 239, 239}{\color[HTML]{000000} Qwen2.5-3B-Instruct}}   \\ \midrule
Base model               & 42.52                                         & 53.96                                 & 1.00                                        & 44.44                                            & 82.49                                              & \cellcolor[RGB]{239, 243, 248}33.04          \\
SFT400                   & 69.29                                         & 41.40                                 & 0.00                                        & \underline{ 94.44}                                      & 60.14                                              & \cellcolor[RGB]{239, 243, 248}34.08          \\
SFT400+PPO               & 78.29                                         & 58.76                                 & 5.12                                        & \textbf{100.00}                                  & 48.40                                              & \cellcolor[RGB]{239, 243, 248}45.80          \\
SFT400+GRPO              & 76.21                                         & 64.15                                 & 1.75                                        & \underline{ 94.44}                                      & \textbf{58.63}                                     & \cellcolor[RGB]{239, 243, 248}46.42          \\
PPO, Cold Start          & \underline{ 82.42}                                   & 67.78                                 & 4.88                                        & \textbf{100.00}                                  & 18.09                                              & \cellcolor[RGB]{239, 243, 248}51.15          \\
ToolRL+GRPO  & 81.58                                         & \textbf{73.78}                        & 3.75                                        & \textbf{100.00}                                  & 56.44                                              & \cellcolor[RGB]{239, 243, 248}52.98          \\
ToolRL+DAPO & 82.19                                         & 69.43                                 & \underline{ 8.00}                                  & 81.25                                            & \underline{ 57.60}                                        & \cellcolor[RGB]{239, 243, 248}\underline{ 53.21}    \\ \midrule
ToolRL+DiPO  & \textbf{83.42}                                & \underline{ 73.06}                           & \textbf{8.62}                               & \textbf{100.00}                                  & 54.16                                              & \cellcolor[RGB]{239, 243, 248}\textbf{55.03} \\\midrule
\multicolumn{7}{c}{\cellcolor[RGB]{239, 239, 239}Qwen2.5-7B-Instruct}           \\ \midrule
Base model      & 66.02                                         & 53.51                                 & 4.25                                        & 76.47                                            & 62.66                                              & \cellcolor[RGB]{239, 243, 248}41.97          \\
SFT400          & 69.29                                         & 41.40                                 & 0.00                                        & \underline{ 94.44}                                      & 8.11                                               & \cellcolor[RGB]{239, 243, 248}34.08          \\
SFT400+PPO      & 83.90                                         & 51.84                                 & 0.25                                        & \textbf{100.00}                                  & 29.66                                              & \cellcolor[RGB]{239, 243, 248}42.02          \\
SFT400+GRPO     & 80.69                                         & 46.51                                 & 0.25                                        & \textbf{100.00}                                  & 14.19                                              & \cellcolor[RGB]{239, 243, 248}39.25          \\
PPO, Cold Start & 79.33                                         & 63.17                                 & 0.38                                        & 88.89                                            & 52.92                                              & \cellcolor[RGB]{239, 243, 248}46.68          \\
ToolRL+GRPO     & 86.17                                         & 74.90                                 & 18.12                                       & 83.33                                            & \textbf{76.68}                                     & \cellcolor[RGB]{239, 243, 248}58.38          \\
ToolRL+DAPO     & \textbf{87.10}                                & \underline{ 76.31}                           & \underline{ 19.75}                                 & 87.50                                            & 67.25                                              & \cellcolor[RGB]{239, 243, 248}\underline{ 61.06}    \\ \midrule
ToolRL+DiPO     & \underline{ 86.21}                                   & \textbf{76.83}                        & \textbf{24.50}                              & 87.50                                            & \underline{ 69.57}                                        & \cellcolor[RGB]{239, 243, 248}\textbf{62.51} \\
\bottomrule
\end{tabular}
}
\label{tab: bfclv3}
\end{table*}

\subsection{Comparison Results of Mathematical Reasoning}

Table \ref{tab: math} shows the comprehensive comparison results of mathematical reasoning in ACC/mean@8 across 6 benchmarks evaluated on three base models, our proposed DiPO demonstrates consistent and superior performance enhancements over other reinforcement learning algorithms. Overall, DiPO achieves the highest average score (AVG) across all three model scales, with 50.55\% for Qwen3-4B-Base, 54.79\% for Qwen3-8B-Base, and 43.56\% for Qwen2.5-7B, surpassing all compared baselines. Specifically, on the more challenging AIME benchmarks, DiPO attains the best results in most cases, such as 29.17\% on AIME24 and 24.58\% on AIME25 with Qwen3-4B-Base, and notably 35.00\% on AIME24 and 27.50\% on AIME25 with Qwen3-8B-Base, indicating its strong capability in handling complex mathematical reasoning tasks. Furthermore, the performance gains are particularly pronounced in larger models like Qwen3-8B-Base, where DiPO outperforms the second-best method by a clear margin in AVG (54.79\% vs. 53.90\%), highlighting its scalability and effectiveness in leveraging model capacity.  Notably, although DAPO with EL achieves highly competitive results, it demonstrates pronounced sensitivity to its configuration coefficients, as further examined in the {Appendix} \ref{sec: app coef}.



\subsection{Comparison Results of Function Calling}

Building upon the mathematical reasoning evaluation, we further assess the performance of DiPO on the function calling task. As presented in the Table \ref{tab: bfclv3}, we use TooRL as the baseline and replicated TooRL+DAPO for comparison. The results demonstrate that DiPO delivers superior overall performance, achieving the highest Overall acc of 55.03\% and 62.51\% on the Qwen2.5-3B-Instruct and Qwen2.5-7B-Instruct models, respectively. Notably, DiPO demonstrates exceptional capability in handling complex, multi-round interactions, as evidenced by its leading performance in Multi-Turn Acc. It achieves scores of 8.62\% and 24.50\% for the 3B and 7B models respectively, surpassing the second-best method (ToolRL+DAPO, with 8.00\% and 19.75\%) by 0.62 and 4.75 percentage points. This comprehensive improvement over strong function calling baselines, validates the effectiveness and general applicability of the DiPO not only in mathematical reasoning.



\subsection{Hyperparameter Analysis}

\begin{table*}[]
\centering
\caption{The impact of hyperparameter $\alpha$ on performance. The best results are marked in \textbf{bold}.}
\scalebox{.87}{
\setlength{\tabcolsep}{4mm}
\begin{tabular}{c|c|c|c|c|c|c|c}
\toprule
\textbf{$\alpha$}     & {\textbf{AIME24}} & {\textbf{AIME25}} & {\textbf{MATH}} & {\textbf{AMC}} & {\textbf{OLY}} & {\textbf{MIN}} & {\textbf{AVG}}    \\ \midrule
\multicolumn{8}{c}{\cellcolor[RGB]{239, 239, 239}Qwen3-4B-Base}     \\ \midrule
$\alpha=0.0$ & 26.25 & 23.75 & 86.43 & 61.90       & 53.88       & 44.34       & \cellcolor[RGB]{239, 243, 248}49.43   \\
$\alpha=0.1$ & \textbf{29.17} & \textbf{24.58} & \textbf{86.93}    & \textbf{63.70}   & \textbf{54.09}   & 44.76       & \cellcolor[RGB]{239, 243, 248}\textbf{50.55} \\
$\alpha=1.0$ & 25.83 & 24.17 & 86.38 & 62.05       & 54.01       & \textbf{45.50}   & \cellcolor[RGB]{239, 243, 248}49.66   \\ \midrule
\multicolumn{8}{c}{\cellcolor[RGB]{239, 239, 239}Qwen3-8B-Base}     \\ \midrule
$\alpha=0.0$ & 30.08 & 25.83 & 89.43 & 69.12       & 56.90       & 48.02       & \cellcolor[RGB]{239, 243, 248}53.23   \\
$\alpha=0.1$ & \textbf{35.00} & \textbf{27.50} & \textbf{89.55}    & \textbf{71.23}   & \textbf{57.73}   & 47.79       & \cellcolor[RGB]{239, 243, 248}\textbf{54.79} \\
$\alpha=1.0$ & 32.08 & 24.58 & 88.68 & 70.03       & 56.71       & \textbf{48.07}   & \cellcolor[RGB]{239, 243, 248}53.36   \\ \bottomrule
\end{tabular}
}
\label{tab: hyperparameter}
\end{table*}
The orthogonal design between the reallocated rewards and validation rewards enables precise, independent adjustment of DiPO's influence via the hyperparameter $\alpha$. Table \ref{tab: hyperparameter} presents the performance variations of Qwen3-4B-Base and Qwen3-8B-Base models under different values of $\alpha$. For both model scales, the performance tends to reach the optimal level when $\alpha$ is set to $0.1$: 4B model achieves the highest average (AVG) performance of 50.55\%, with significant improvements compared to $\alpha=0.0$ and $\alpha=1.0$; similarly, 8B-model attains the best AVG of 54.79\% at $\alpha=0.1$. When $\alpha=0.0$, DiPO's contribution is zero, leading to the lowest AVG performance for both models (49.43\% and 53.23\%).


\subsection{Ablation Experiment and Discussion}

Table \ref{tab: ablation study} presents ablation results of the contributions of PSD and BRR on Qwen3-4B-Base and Qwen3-8B-Base on six mathematical reasoning benchmarks. PPL reward refers to using PPL directly as rewards: positive PPL for hard groups, and negative PPL for easy groups. Overall, the combination of PSD and BRR achieves the best performance for both 4B and 8B models, verifying their effectiveness. Detailed discussions are as follows.

\textbf{Discussion 1. Fine-grained exploration and exploitation make RL more effective.} In Table \ref{tab: ablation study}, without PSD, all hard samples are indiscriminately driven toward high PPL, whereas easy samples are steered toward low PPL. PSD improves this by disentangling the space into ErS (high PPL) and EiS (low PPL), allowing for fine-grained optimization: easy groups within ErS are encouraged toward EiS, while hard groups within EiS are directed toward ErS. The results confirm the critical role of PSD in model enhancement. When using the PPL reward, PSD yields improvements of 2.39 and 1.19 points on the 4B and 8B models, respectively; with the BRR reward, the corresponding gains are even larger at 2.99 and 3.88 points. Significant improvements indicate that invalid exploration and exploitation are needed for optimization, while excessive encouragement of exploration and exploitation is detrimental.

\textbf{Discussion 2. Reward shaping should not cause significant changes to the intrinsic validation rewards.} In Table \ref{tab: ablation study}, we compare our BRR against PPL reward. The experimental results confirm that BRR is superior to the PPL reward method. Concretely, with using PSD, BRR compared to PPL reward brings an improvement of 0.93 and 2.23 respectively on the 4B and 8B models, while using PSD and PPL reward is even worse than the baseline. Science the reward variance of easy and hard groups are both zeros, directly using PPL as a reward causes drastic shifts in the reward distribution, potentially destabilizing the training process. In contrast, BRR introduces the reallocated rewards maintain a distribution very similar to the original validation rewards, which allows BRR to leverage PPL signals for enhancement with the minimum impact on the validation reward landscape.

\begin{table*}[]
\centering
\caption{Ablation study results on Qwen3-4B-Base and Qwen-8B-Base, showing ACC/mean@8 on six mathematical reasoning benchmarks. The best results are marked in \textbf{bold}.}
\scalebox{.87}{
\setlength{\tabcolsep}{1.3mm}
\begin{tabular}{c|c|c|c|c|c|c|c|c|c}
\toprule
\quad\textbf{PSD}\quad & \quad\textbf{BRR}\quad & \textbf{PPL reward} & \quad\textbf{AIME24}\quad & \quad\textbf{AIME25}\quad & \quad\textbf{MATH}\quad  & \quad\textbf{AMC}\quad   & \quad\textbf{OLY}\quad   & \quad\textbf{MIN}\quad   & \quad\textbf{AVG}\quad \\ \midrule
\multicolumn{10}{c}{\cellcolor[RGB]{239, 239, 239}Qwen3-4B-Base}                                                                                                                                         \\ \midrule
\textcolor{red}{\ding{55}} & \textcolor{red}{\ding{55}} & \textcolor{red}{\ding{55}} & 26.25 & 23.75 & 86.43 & 61.90 & 53.88 & 44.34 & \cellcolor[RGB]{239, 243, 248}49.43 \\
\textcolor{red}{\ding{55}} & \textcolor{green}{\ding{51}} & \textcolor{red}{\ding{55}} & 25.83 & 20.83 & 85.10 & 59.19 & 49.37 & 43.06 & \cellcolor[RGB]{239, 243, 248}47.23 \\
\textcolor{red}{\ding{55}} & \textcolor{red}{\ding{55}} & \textcolor{green}{\ding{51}} & 24.58 & 23.75 & 84.60 & 61.45 & 49.09 & 41.82 & \cellcolor[RGB]{239, 243, 248}47.55 \\
\textcolor{green}{\ding{51}} & \textcolor{red}{\ding{55}} & \textcolor{green}{\ding{51}} & 27.50 & 23.33 & 86.35 & 62.35 & 54.00 & 44.21 & \cellcolor[RGB]{239, 243, 248}49.62 \\
\textcolor{green}{\ding{51}} & \textcolor{green}{\ding{51}} & \textcolor{red}{\ding{55}} & \textbf{29.17}  & \textbf{24.58}  & \textbf{86.93} & \textbf{63.70} & \textbf{54.09} & \textbf{44.76} & \cellcolor[RGB]{239, 243, 248}\textbf{50.55} \\ \midrule
\multicolumn{10}{c}{\cellcolor[RGB]{239, 239, 239}Qwen3-8B-Base}                                                                                                                                         \\ \midrule
\textcolor{red}{\ding{55}} & \textcolor{red}{\ding{55}} & \textcolor{red}{\ding{55}} & 30.08 & 25.83 & 89.43 & 69.12 & 56.90 & \textbf{48.02} & \cellcolor[RGB]{239, 243, 248}53.23 \\
\textcolor{red}{\ding{55}} & \textcolor{green}{\ding{51}} & \textcolor{red}{\ding{55}} & 30.08 & 24.58 & 88.35 & 65.36 & 53.62 & 46.23 & \cellcolor[RGB]{239, 243, 248}51.37 \\
\textcolor{red}{\ding{55}} & \textcolor{red}{\ding{55}} & \textcolor{green}{\ding{51}} & 27.92 & 25.42 & 88.00 & 64.76 & 53.45 & 45.63 & \cellcolor[RGB]{239, 243, 248}50.86 \\
\textcolor{green}{\ding{51}} & \textcolor{red}{\ding{55}} & \textcolor{green}{\ding{51}} & 30.83 & 24.58 & 88.15 & 68.83 & 56.65 & 46.30 & \cellcolor[RGB]{239, 243, 248}52.56 \\
\textcolor{green}{\ding{51}} & \textcolor{green}{\ding{51}} & \textcolor{red}{\ding{55}} & \textbf{35.00}  & \textbf{27.50}  & \textbf{89.55} & \textbf{71.23} & \textbf{57.73} & 47.42 & \cellcolor[RGB]{239, 243, 248}\textbf{54.79} \\ \bottomrule
\end{tabular}
}
\label{tab: ablation study}
\end{table*}

\subsection{Quantitative Results}
To further verify the properties of DiPO, we conducted a quantitative analysis of some data during the training process of DAPO and DiPO.

\textbf{Exploration-exploitation trade-off.} Figure \ref{fig: ppl distribution} shows the PPL distribution of the Qwen3-8B-Base after training on DAPO-17K with DAPO and DiPO. The results demonstrate that under the DAPO training, the PPL distribution of correct and incorrect samples exhibits a significant overlap. This phenomenon directly causes a large number of hard samples to lose exploration capability, thereby limiting the plasticity of RL. In contrast, DiPO shows superior distribution characteristics: error samples are more located in the high-PPL region, while correct samples are still concentrated in the low-PPL region, which achieves a balanced trade-off between exploration and exploitation in RL.

\begin{figure}[h]
\centering
\begin{minipage}{0.48\linewidth}
    \centering
    \subfigure{\includegraphics[width=0.48\linewidth]{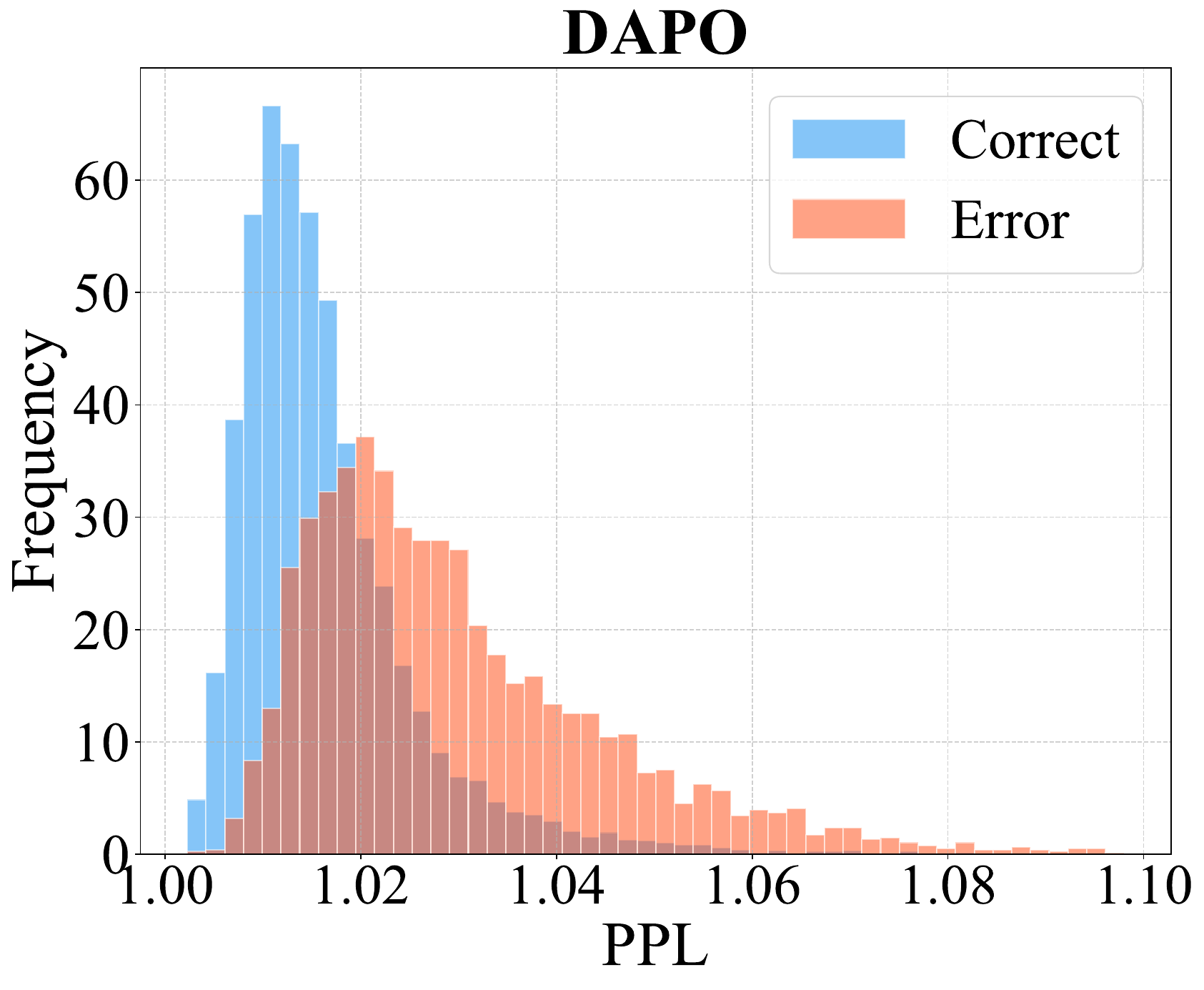}}
    \subfigure{\includegraphics[width=0.48\linewidth]{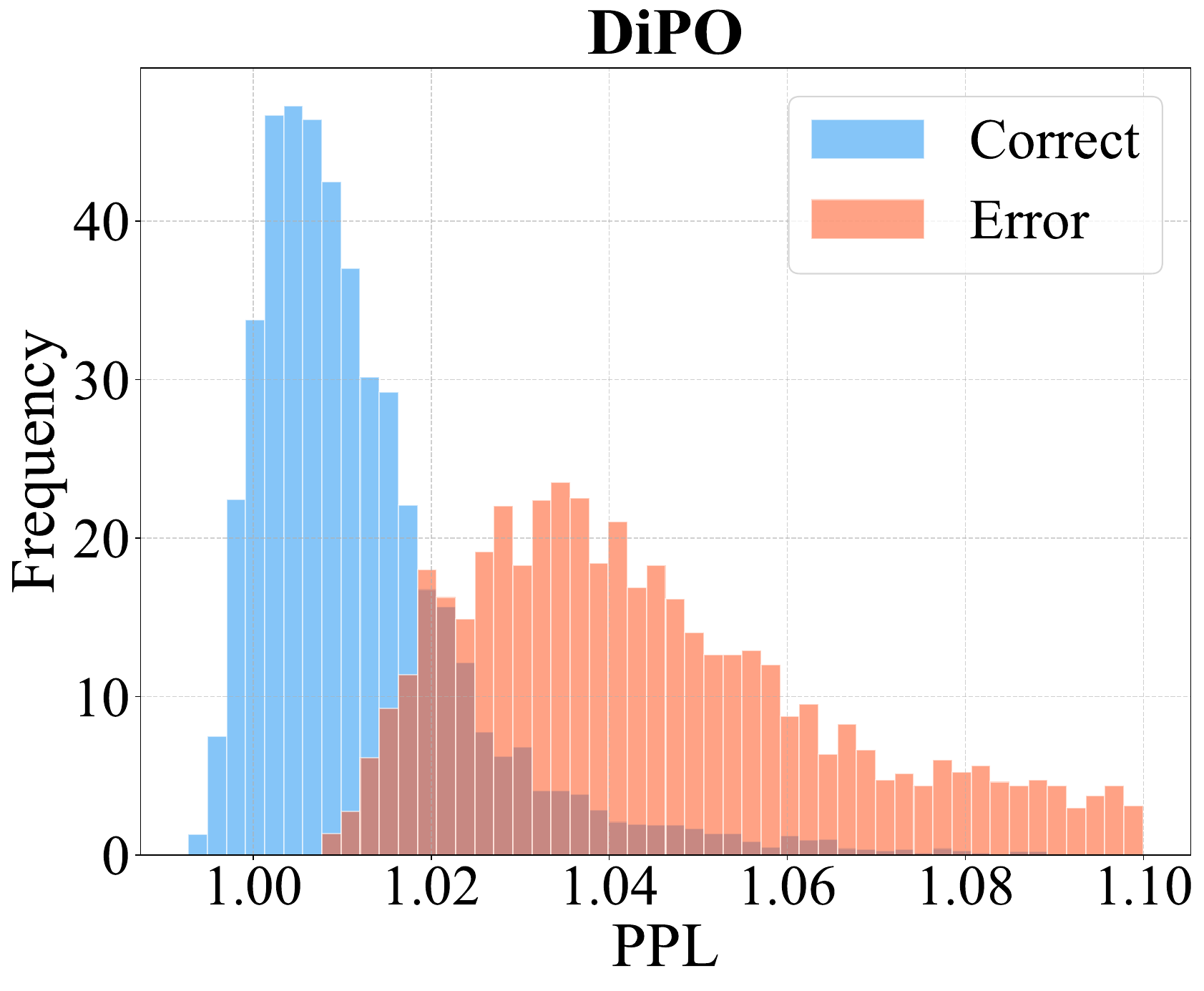}}
    \caption{PPL distribution comparison of correct and error samples for Qwen3-8B-Base trained on DAPO-17K Dataset via DAPO and DiPO.}
    \label{fig: ppl distribution}
\end{minipage}
\hspace{0.01\linewidth}  
\begin{minipage}{0.48\linewidth}
    \centering
    \subfigure{\includegraphics[width=0.48\linewidth]{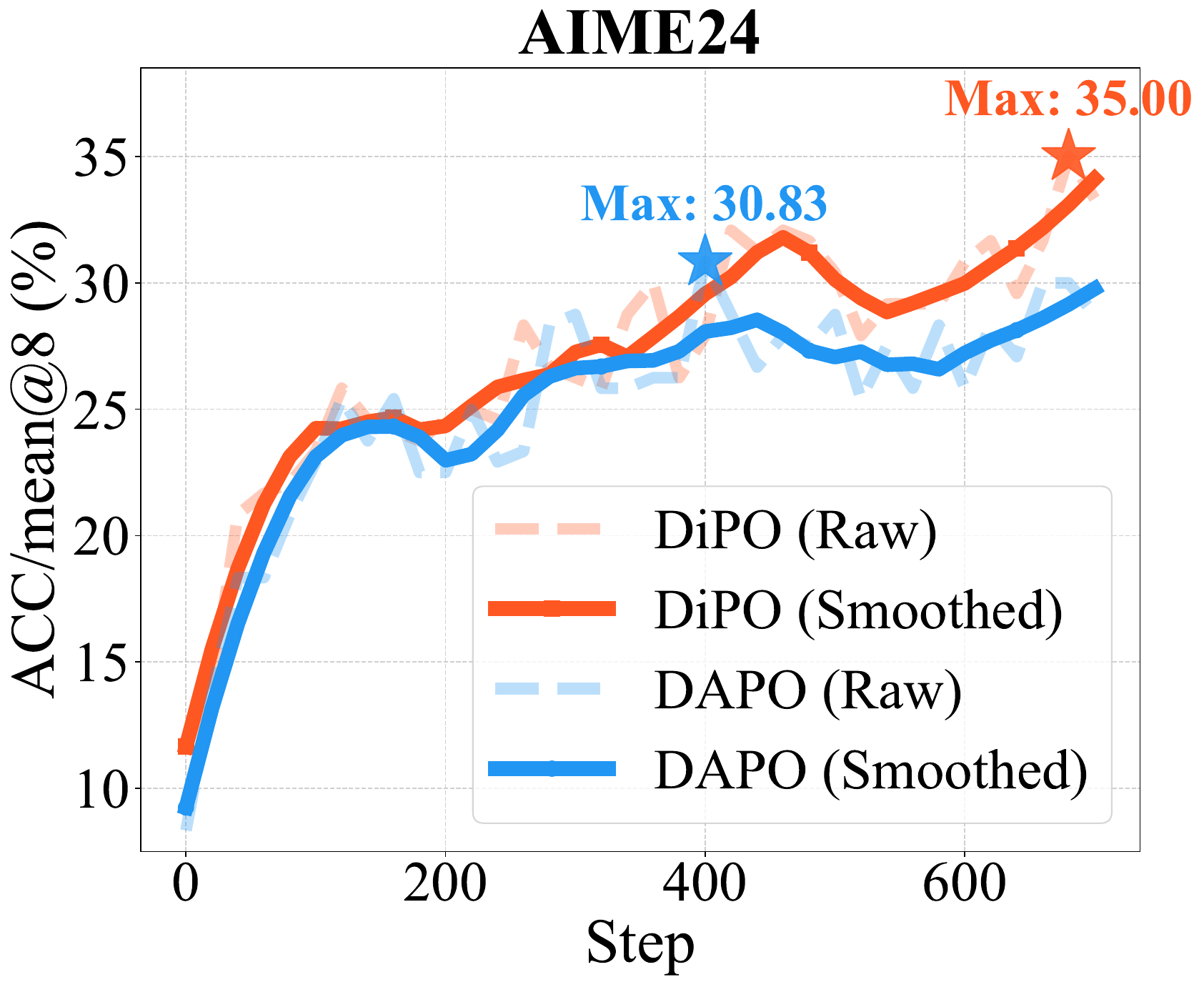}}
    \subfigure{\includegraphics[width=0.48\linewidth]{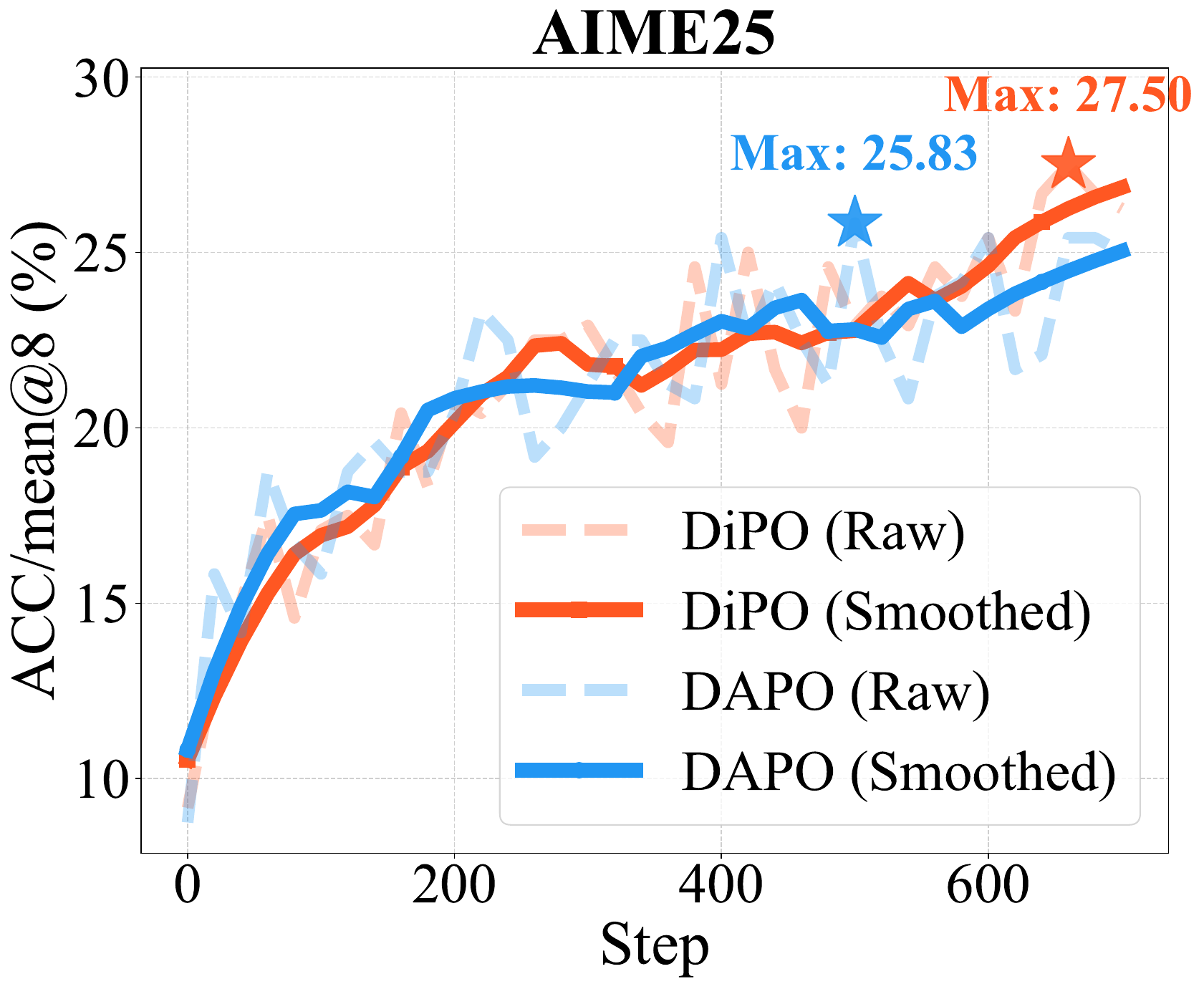}}
    \caption{ACC/mean@8 curves of DiPO and DAPO (raw and smoothed curves) on AIME24 and AIME25 with using Qwen3-8B-Base model.}
    \label{fig: acc curves}
\end{minipage}
\end{figure}


\textbf{Higher upper bound for later training.} Figure \ref{fig: acc curves} presents the test curves of DiPO and DAPO on two AIME benchmarks during the training process. The results show that at the initial stage of training, there was no significant difference in the performance of the two methods; in the later stages of training, however, the performance curve of DAPO exhibited a significant slowdown in growth, whereas DiPO maintained a sustained capacity for exploratory improvement. Combined with the PPL distribution characteristics in Figure \ref{fig: acc curves}, this phenomenon can be further explained: DiPO enables hard samples to conduct more sufficient exploration, while the training paradigm of DAPO leads to a conservative tendency in the exploration strategy of hard samples.


\section{Conclusion}
This paper explores the exploration-exploitation trade-off between in RL training. We analyze the two ETTO dilemmas faced by the extreme hard and easy groups, and proposed the Disentangled Perplexity Policy Optimization. First, we develop a novel Perplexity Space Disentangling method to identify which hard samples require encouragement for exploration and which easy samples need promotion for exploitation. Then, we design a Bidirectional Reward Reallocation mechanism that incorporates PPL-based exploration and exploitation signals while minimizing disruptions to the original verification reward distribution. Finally, extensive experiments on mathematical reasoning and function calling tasks demonstrate the comprehensive superiority of the proposed method.

\bibliographystyle{plainnat}   
\bibliography{ref}              

\newpage

\appendix

\section{Proof of Section~\ref{sec: brr}}
\label{sec: proof}
\subsection{Mathematical proof}
\begin{proof}
Given a language model $\pi_\theta$ with parameters $\theta$. Given a query $q$, sample $G$ responses $\{o^i\}_{i=1}^G$. Each response $o^i$ is a token sequence $(o^i_1, \dots, o^i_{T^i})$. The model's probability for token $o^i_t$ given context $(q, o^i_{<t})=\pi_\theta(y \mid q, o^i_{<t})$ is denoted as  $\pi_t^i(o^i_t)$. The probability of other tokens in position of the $t$-th token is denoted as $\pi_t^i(y)$, and the logistic of $y$ is denoted as $l_i^t(y)$, $y\in \mathcal{V}$ and $\mathcal{V}$ is the vocabulary of $\pi_\theta$.
For convenience, we do not consider the clip operation, the loss function is as follows:
\begin{equation}
    \begin{aligned}
    L(\theta)  = -\frac{1}{G}\sum_{i=1}^G \frac{1}{|o^i|}\sum_{t=1}^{|o^i|} \hat{A}^i_t \frac{\pi^i_t(o^i_t)}{\pi_{old}(o^i_t)} , \   \text{where} \ \ \ \ \pi_t^i(o^i_t) = \frac{exp(l^i_t(o_i^t))}{\sum_{y \in \mathcal{V}}exp(l^i_t(y))}.
    \end{aligned}
    \label{eq: loss}
\end{equation}
Optimization gradient descent is $\theta' = \theta - \eta \nabla_\theta L$, with small learning rate $\eta$.
Define the token-level entropy for response $i$ at token $t$:
\begin{equation}
    H^i_t = -\sum_{y \in \mathcal{V}} \pi_\theta(y \mid q, o^i_{<t}) \log \pi_\theta(y \mid q, o^i_{<t}) = -\sum_{y \in \mathcal{V}} \pi_t^i(y) \log \pi_t^i(y).
    \label{eq: entropy}
\end{equation}
The average token entropy over all responses and time steps is:
\begin{equation}
    H_{\text{avg}} = \frac{1}{G}  \sum_{i=1}^G \frac{1}{|o^i|} \sum_{t=1}^{|o^i|} H^i_t.
    \label{eq: entropy_avg}
\end{equation}
\emph{The following will prove that high PPL rewards and high PPL penalties can respectively increase and decrease $H_{\text{avg}}$.}

\noindent \textbf{Logit Change.} Consider a single token of response $o^i_t$ for training, the loss contribution is $-\hat{A}^i_t \frac{\pi^i_t(o^i_t)}{\pi_{old}(o^i_t)}$. The gradient of the loss with respect to logit $l^i_t(y)$ is:
\begin{equation}
    \frac{\partial (-\hat{A}^i_t \frac{\pi^i_t(y)}{\pi_{old}(y)})}{\partial l_t^i(y)} 
    = \hat{A}^i_t \frac{\pi^i_t(y)}{\pi_{old}(y)} (\pi^i(y) - \mathbf{1}_{y=o^i_t}) \approx \hat{A}^i_t (\pi^i(y) - \mathbf{1}_{y=o^i_t}).
    \label{eq: partial_loss}
\end{equation}
Using gradient descent with learning rate $\eta$, the logit update is:
\begin{equation}
    \Delta l_t^i(y) = \eta \hat{A}^i_t (\mathbf{1}_{y=o^i_t} - \pi_t^i(y)).
    \label{eq: delta_logi}
\end{equation}
\textbf{Probability Change.} According the multivariate Taylor Expansion, for softmax, the first-order change in probability is:
\begin{equation}
    \label{eq: delta_p}
\resizebox{1\hsize}{!}{$
    \begin{aligned}
        \Delta \pi_t^i(y) 
        \approx \sum_{z \in \mathcal{V}} \frac{\partial \pi_t^i(y)}{\partial l^i_t(z)} \Delta l_t^i(z) 
        = \sum_{z \in \mathcal{V}} \pi_t^i(y) (\mathbf{1}_{z=y} - \pi_t^i(z) )  \Delta l_t^i(z) 
        = \pi_t^i(y) \left( \Delta l_t^i(y) - \sum_{z \in \mathcal{V}} \pi_t^i(z) \Delta l_t^i(z) \right)
    \end{aligned}
$}
\end{equation}
where $\pi_t^i(y) = \frac{exp(l^i_t(y))}{\sum_{z \in \mathcal{V}}exp(l^i_t(z))}$. Substituting $\Delta \pi_t^i(y)$.
\begin{equation}
    \begin{aligned}
    \Delta \pi_t^i(y) 
    &= \pi_t^i(y) \left( \eta \hat{A}^i_t (\mathbf{1}_{y=o^i} - \pi_t^i(y)) - \sum_{z \in \mathcal{V}} \pi_t^i(z) \eta \hat{A}^i_t (\mathbf{1}_{z=o^i_t} - \pi_t^i(z)) \right) \\
    &= \eta \hat{A}^i_t \pi^i_t(y) \left( \mathbf{1}_{y=o^i_t} - \pi^i_t(y) - \pi^i_t(o^i_t) + \sum_{z\in \mathcal{V}} \pi^i_t(z)^2 \right).
    \end{aligned}
    \label{eq: delta_p2}
\end{equation}

\noindent \textbf{Entropy Change.} According the Taylor Expansion, the first-order change in entropy $H^i_t$ is:
\begin{equation}
\resizebox{1\hsize}{!}{$
    \begin{aligned}
        \Delta H^i_t 
        & \approx \sum_{y\in \mathcal{V}} \frac{\partial H_i^t}{\partial \pi_t^i(y)} \Delta \pi_t^i(y)
        = -\sum_y \Delta \pi^i_t(y) (1 + \log \pi^i_t(y)), 
        & \text{where} \ \ \ H_i^t= \left(-\sum_{y \in \mathcal{V}} \pi_t^i(y) \log \pi_t^i(y) \right).
    \end{aligned}
$}
\end{equation}
Substituting $\Delta \pi^i_t(y)$:
\begin{equation}
\resizebox{1\hsize}{!}{$
    \begin{aligned}
        \Delta H_t^i 
        & = -\eta \hat{A}^i_t \sum_{y\in \mathcal{V}} \pi^i_t(y) \Bigg[ \mathbf{1}_{y=o^i_t} - \pi^i_t(y) - \pi^i_t(o^i_t) + \sum_{z\in \mathcal{V}} \pi^i_t(z)^2 \Bigg] (1 + \log \pi^i_t(y)) \\
        = & -\eta \hat{A}^i_t \Bigg[ \pi^i_t(o^i_t)\log \pi^i_t(o^i_t) - \sum_{y\in \mathcal{V}} \pi^i_t(y)^2\log \pi^i_t(y) - \pi^i_t(o^i_t)\sum_{y\in \mathcal{V}} \pi^i_t(y)\log \pi^i_t(y) + \sum_{z\in \mathcal{V}}  \pi^i_t(z)^2 \sum_{y\in \mathcal{V}} \pi^i_t(y) \log \pi^i_t(y) \Bigg] \\
        = & -\eta \hat{A}^i_t \Bigg[ \pi^i_t(o^i_t)\log \pi^i_t(o^i_t) - \sum_{y\in \mathcal{V}}  \pi^i_t(y)^2 H^i_t - \sum_{y\in \mathcal{V}} \pi^i_t(y)^2\log \pi^i_t(y) + \pi^i_t(o^i_t)H^i_t \Bigg] \\  
    \end{aligned}
$}
\end{equation}

\noindent \textbf{Average Entropy Change.} \emph{Ignore the impact of cross-context updates}, the first-order change in token entropy over all queries is:
\begin{equation}
\resizebox{1\hsize}{!}{$
    \begin{aligned}
        \Delta H_{avg} = \mathbb{E}_{q\sim \mathcal{D}}\Bigg[\frac{1}{G} \sum_{i=1}^G \frac{1}{|o^i|}\sum_{t=1}^{|o^i|} -\eta \hat{A}^i_t \Bigg( \pi^i_t(o^i_t)\log \pi^i_t(o^i_t)
        - \sum_{y\in \mathcal{V}} \pi^i_t(y)^2\log \pi^i_t(y) + \pi^i_t(o^i_t)H^i_t - \sum_{y\in \mathcal{V}}  \pi^i_t(y)^2 H^i_t \Bigg) \Bigg] \\  
    \end{aligned}
$}
\end{equation}
Let $B_t^i = \pi^i_t(o^i_t)\log \pi^i_t(o^i_t) + \pi^i_t(o^i_t)H^i_t$ and $F_t^i= - \sum_{y\in \mathcal{V}} \pi^i_t(y)^2\log \pi^i_t(y) - \sum_{y\in \mathcal{V}}  \pi^i_t(y)^2 H^i_t$:
\begin{equation}
    \begin{aligned}
        \Delta H_{avg} = \mathbb{E}_{q\sim \mathcal{D}}\left[\frac{-\eta}{G} \sum_{i=1}^G \mathbb{E}_{t=1:|o^i|} \left[\hat{A}^i_t \left( B_t^i + F_t^i \right) \right] \right] \\  
    \end{aligned}
\end{equation}

\noindent The response with maximum PPL of $\{o^i\}_{i=1}^G$ is denoted as $o^m$, 

\begin{itemize}
    \item Rewarding maximum PPL $\hat{A}^m_t = \sqrt{G - 1}, \hat{A}^i_t = -\frac{1}{\sqrt{G - 1}}$.
    \begin{equation}
        \resizebox{0.8\hsize}{!}{$
        \begin{aligned}
            \Delta H_{avg} 
            & = \mathbb{E}_{q\sim \mathcal{D}}\left[\frac{-\eta}{G} \left(-\frac{1}{\sqrt{G - 1}} \sum_{\substack{i=1 \\ i\not=m}}^G \mathbb{E}_{t=1:|o^i|} \left[ \left( B_t^i + F_t^i \right) \right] + \sqrt{G - 1} \mathbb{E}_{t=1:|o^m|} \left[ \left( B_t^m + F_t^m \right) \right] \right) \right] \\
            & = \mathbb{E}_{q\sim \mathcal{D}} \left[ \frac{-\eta\sqrt{G-1}}{G} \left( \mathbb{E}_{t=1:|o^m|}\left[B_t^m + F_t^m \right] - \mathbb{E}_{i\not=m}\left[ \mathbb{E}_{t=1:|o^i|}\left[ B_t^i+F_t^i\right]\right] \right) \right] \\
            & = \frac{-\eta\sqrt{G-1}}{G} \left( \mathbb{E}\left[B_t^m\right] + \mathbb{E}\left[F_t^m \right] - \mathbb{E}_{i\not=m}\left[ B_t^i\right] - \mathbb{E}_{i\not=m}\left[F_t^i\right] \right) \\
        \end{aligned}
        $}
    \end{equation}
    \emph{We assume the difference in $F_t$ is negligible as it reflects global distribution statistics that remain statistically invariant across samples in a large dataset, thus} $\mathbb{E}\left[F_t^m \right] - \mathbb{E}_{i\not=m}\left[F_t^i\right] \approx 0$.
    \begin{equation}
        \begin{aligned}
            \Delta H_{avg} \approx \frac{\eta\sqrt{G-1}}{G} \left(\mathbb{E}_{i\not=m}\left[ B_t^i\right] -\mathbb{E}\left[B_t^m\right] \right) \\
        \end{aligned}
    \end{equation}
    \item Penalizing maximum PPL $\hat{A}^m_t = -\sqrt{G - 1}, \hat{A}^i_t = \frac{1}{\sqrt{G - 1}}$.
        \begin{equation}
        \begin{aligned}
            \Delta H_{avg} \approx \frac{\eta\sqrt{G-1}}{G} \left( \mathbb{E}\left[B_t^m\right] - \mathbb{E}_{i\not=m}\left[ B_t^i\right] \right) \\
        \end{aligned}
    \end{equation}
\end{itemize}
\noindent When $e^{-1-H_t^i}<o_t^i<1$, $B^i_t$ increases with $\pi_t^i(o_t^i)$. Moreover, since $\{o^i\}$ is sampled by the $\pi$ with probability, and the smaller $H_t^i$ is, the higher the sampling probability, we can approximately assume that $e^{-1-H_t^i}<o_t^i<1$ is basically satisfied. \textbf{The conclusion is that rewarding maximum PPL then $\Delta H_{avg}>0$ (entropy increases) and penalize maximum PPL then $\Delta H_{avg}<0$ (entropy decreases)}.
\end{proof}
\begin{mdframed}[
    linecolor=green!70!black, 
    backgroundcolor=green!5!white, 
    roundcorner=20pt, 
    linewidth=0pt, 
    innermargin=10pt, outermargin=10pt, 
    skipabove=5, skipbelow=5 
]
\textbf{Import Statement.} Given that the parameter space of LLMs is vast and complex updating, strict mathematical proof is unrealizable. The aforementioned proof is based on multiple idealized assumptions and only provide an approximate estimate of the trend in entropy change. In actual training, the theoretical results need to be further verified through experiments.
\end{mdframed}
\subsection{Experimental verification}
To further verify the feasibility of the theory in section \ref{sec: brr}, we designed a verification experiment. Specifically, we used the Qwen3-0.6B \cite{qwen3} model and employed DAPO-17K \cite{DAPO} as the training set, discarding the verification reward. Instead, we utilize max-PPL reward and max-PPL penalty as training reward, respectively, and recorded the changes in model entropy. As illustrated in Figure \ref{fig: ppl val}, the trend of entropy update is consistent with the proof.

\begin{figure*}[h]
\centering
  \subfigure{
    \includegraphics[width=0.47\linewidth]{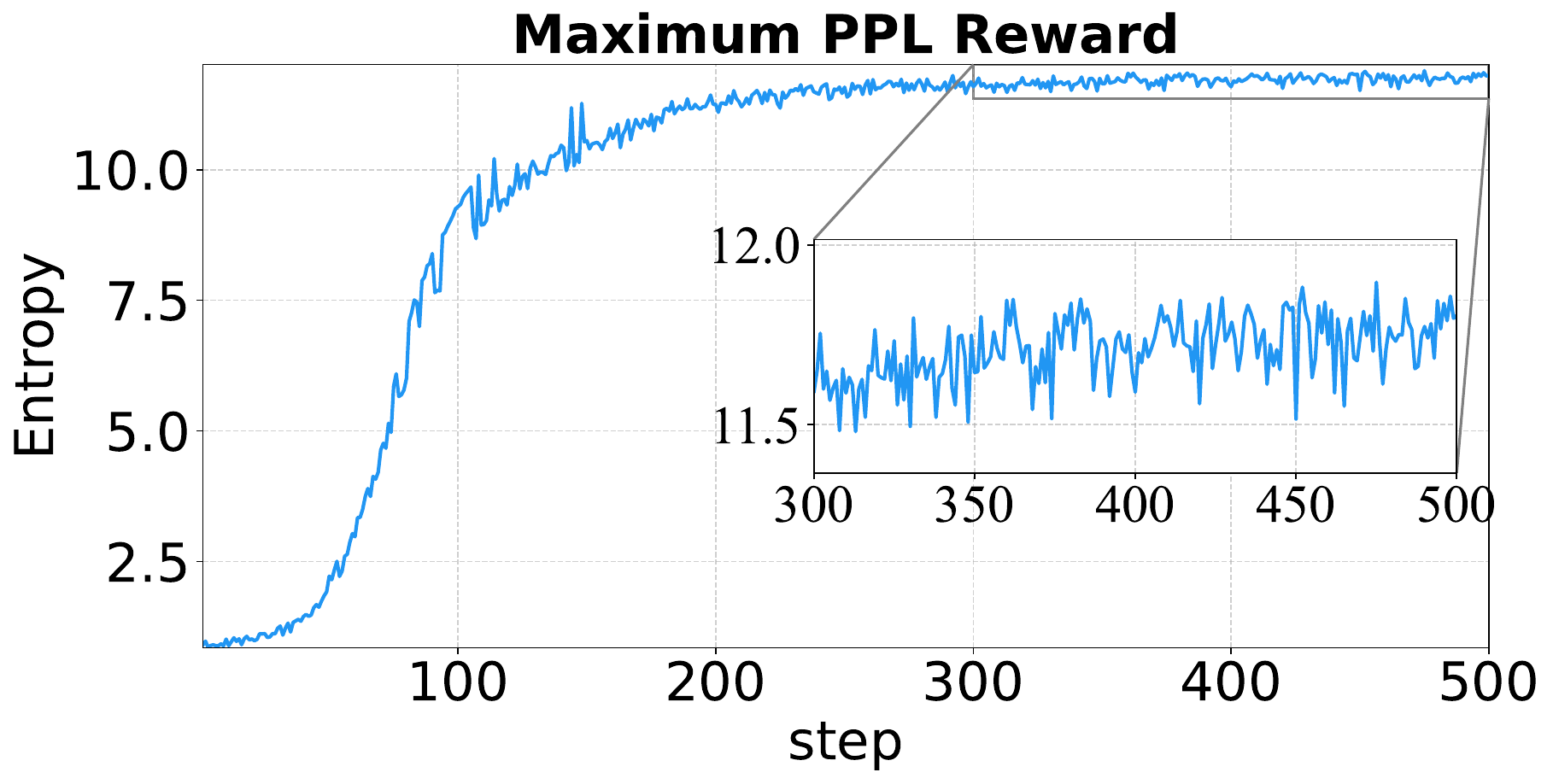} 
  }
  \subfigure{
    \includegraphics[width=0.47\linewidth]{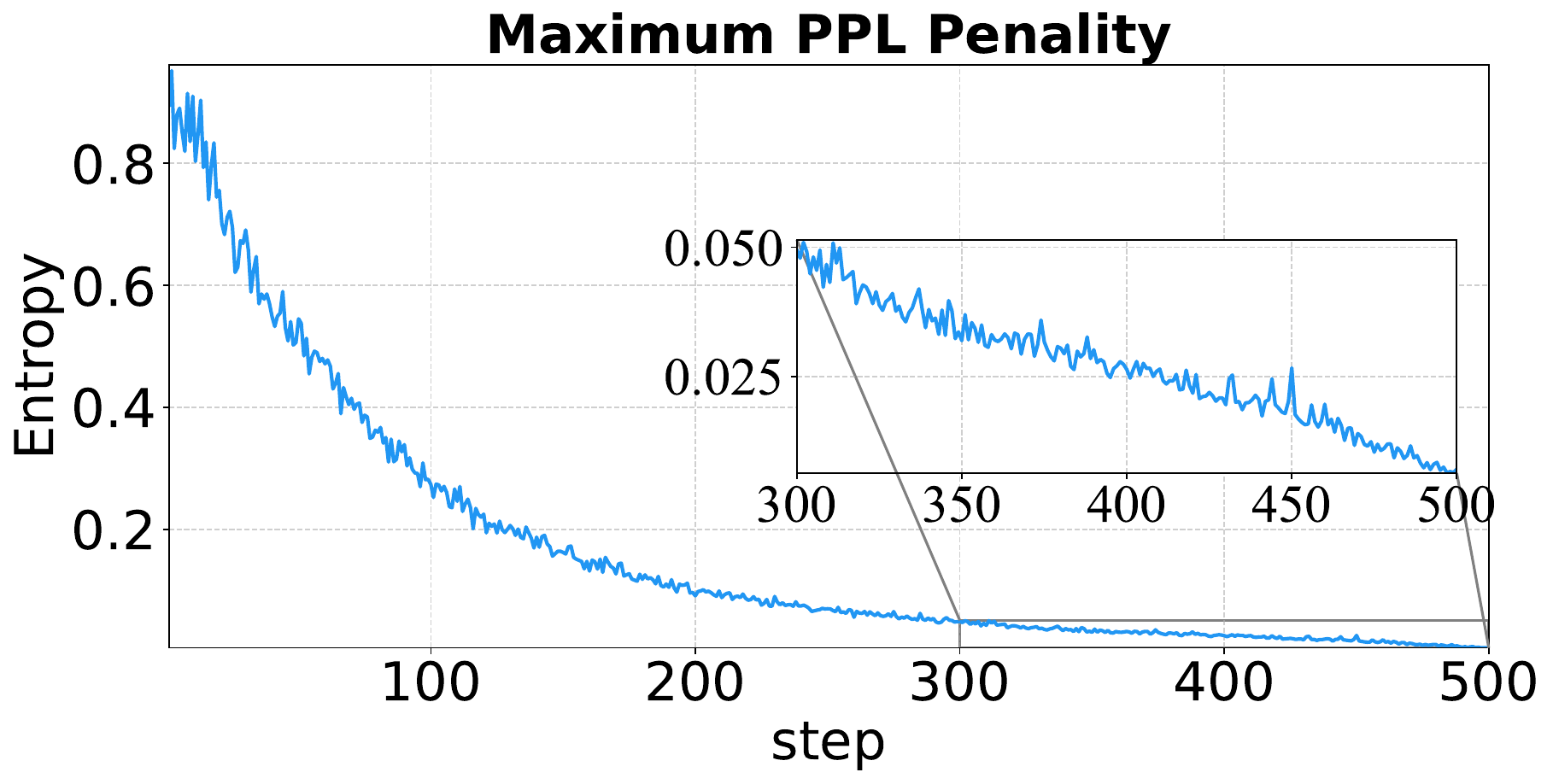}
  }
\caption{Entropy curves of maximum-PPL reward and maximum-PPL penalty trained on Qwen3-0.6B model  with using DAPO-17K.}
\label{fig: ppl val}
\end{figure*}

\section{Related Works}
\subsection{Reinforcement Learning for LLMs}
Policy optimization algorithms have evolved significantly to address the challenges of RL for LLMs (LLM-RL). Trust Region Policy Optimization (TRPO) \cite{TRPO} introduced KL divergence constraints for stable policy updates, laying foundational ideas. Proximal Policy Optimization (PPO) \cite{PPO} simplified TRPO with clip-based objectives, becoming the standard for LLM-RL. GRPO \cite{GRPO} optimized PPO for mathematical reasoning by a novel advantage calculation strategy, boosting efficiency and removing value networks. DAPO \cite{DACE} further scaled LLM-RL with decoupled clipping and dynamic sampling, while VAPO \cite{VAPO} enhanced reasoning reliability via hierarchical advantage estimation. GSPO \cite{GSPO} extended GRPO's grouping strategy to sequence-level optimization for mixture-of-experts models and long-form reasoning tasks.

\subsection{Reward Shaping}
Recent works on reward shaping for LLM-RL have advanced across key directions. CrossDomain-RLVR \cite{CrossDomain-RLVR} expanded verifiable reward RL to diverse unstructured domains via generative scoring. PKPO \cite{PKPO} and Pass@k Training \cite{Pass@kTraining} optimized pass@k performance to enhance sample diversity and collective utility, addressing exploration limitations. rl-without-gt \cite{rl-without-gt} introduced format-length surrogate signals to bypass ground truth dependence. RuscaRL \cite{RuscaRL} leveraged rubric scaffolding to break exploration bottlenecks, while DACE \cite{DACE}, CDE \cite{CDE} proposed difficulty-aware certainty and curiosity-driven signals for adaptive exploration. DRER \cite{DRER} focused on reasoning quality with fine-grained CoT rewards, and OBE \cite{OBE} mitigated diversity collapse via outcome-based exploration bonuses. 

\subsection{Entropy and Perplexity-Driven Exploration}
Entropy and perplexity are key signals for RL of adaptive balance between exploration and exploitation. Some advantage-enhanced methods \cite{EntropyAdvantage, trial} revisited entropy/perplexity as a signal, augmenting the advantage function to promote deep reasoning chains and boost Pass@K. \cite{ETTRL} extended entropy mechanisms to test-time RL via ETMR and EAR, enhancing efficiency and diversity. \cite{CDE} integrated actor perplexity and critic value variance as curiosity bonuses to mitigate premature convergence. \cite{DACE} leveraged difficulty-aware certainty to dynamically modulate exploration, rewarding or penalizing confidence based on task complexity. \cite{EEPO} introduced adaptive unlearning in two-stage rollouts to break entropy collapse loops. 

\section{Algorithm Details}
\begin{algorithm}[]
\caption{Perplexity Space Disentangling}
\label{alg: psd}
\begin{algorithmic}[1]
\REQUIRE Batch of queries $\{\boldsymbol{q}\}$ from $\mathcal{D}$, Policy $\pi_{\boldsymbol{\theta}}$, PPL Queue $\mathcal{Q}$
\FOR{each query $\boldsymbol{q}$}
    \STATE \textbf{// G-Sampling}
    \STATE Generate outputs $\{\boldsymbol{o}^i\}_{i=1}^G \sim \pi_{\boldsymbol{\theta}}(\cdot|\boldsymbol{q})$
    \STATE Compute verifiable rewards $\{\boldsymbol{r}^i\} \leftarrow \mathcal{R}(\boldsymbol{o}^i, \boldsymbol{a})$
    \STATE Compute perplexity $\{\boldsymbol{p}^i\}$ using Eq. (\ref{eq: ppl})
    \STATE Update $\mathcal{Q}$ with pairs $\{(\boldsymbol{p}^i, \boldsymbol{r}^i)\}$
\ENDFOR
\STATE Estimate probabilities $\widehat{\Pr}(R|P > \tau)$ and $\widehat{\Pr}(R|P < \tau)$ using Eq. (\ref{eq:conditional_probs})
\STATE Calculate candidate set $\mathcal{C}=\{\tau|\Delta_{\mathrm{EiS}}(\tau)>0 \land\Delta_{\mathrm{ErS}}(\tau)>0\}$ using (Eq. \ref{eq:correlation_gap} Eq. \ref{eq:ci_gap})
\IF {$\mathcal{C} \not= \emptyset$}
    \STATE $\tau^* \leftarrow \arg\min_{\tau\in\mathcal{C}} \frac{1}{|\mathcal{Q}|} \sum_{\scriptsize{(\boldsymbol{r}_i, \boldsymbol{p}_i) \in \mathcal{Q}}} | \boldsymbol{r}_j - \mathbb{I}(\boldsymbol{p}_j < \tau) |$
\ELSE
    \STATE $\tau^* \leftarrow \text{None}$
\ENDIF
\end{algorithmic}
\end{algorithm}

\begin{algorithm}[]
\caption{Bidirectional Reward Reallocation}
\label{alg: brr}
\begin{algorithmic}[1]
\REQUIRE Optimal threshold $\tau*$, A group of verification reward $\{\boldsymbol{r}^i\}$ and corresponding PPL $\{\boldsymbol{p}^i\}$
\STATE Initialize reallocated rewards $\{\boldsymbol{r}_r^i\} \leftarrow \{\boldsymbol{r}^i\}$
\IF{$\tau^*$ is not None}
    \STATE ${p} = \text{mean}(\{\boldsymbol{p}^i\})$
    \STATE $m = \arg\max_i \boldsymbol{p}^i$
    \IF{$\forall i, \boldsymbol{r}^i = 0$ \textbf{and} ${p} < \tau^*$} 
        \STATE \textbf{// Hard Group in EiS}
        \STATE $\boldsymbol{r}_r^m \leftarrow 1$ \COMMENT{Encourage Exploration}
    \ELSIF{$\forall i, \boldsymbol{r}^i = 1$ \textbf{and} ${p} > \tau^*$}
        \STATE \textbf{// Easy Group in ErS}
        \STATE $\boldsymbol{r}_r^m \leftarrow 0$ \COMMENT{Encourage Exploitation}
    \ENDIF
\ENDIF
\end{algorithmic}
\end{algorithm}

\section{Detailed Experiment Setup}
\subsection{Mathematical Reasoning}
\label{sec: app math}

The mathematical reasoning task uses the DAPO-17K \cite{DAPO} dataset for 700 steps training on 4$\times$8 A100 GPUs. For all comparison methods, all parameters are set from Table \ref{tab: math config}, and the remaining parameters use the default parameters of VERL \cite{verl}. For DAPO/w EL, use 0.001 as the coefficient for entropy loss, and the coefficient for entropy loss for the other methods is all 0. For the reproduced CDE \cite{CDE}, their exclusive hyper-parameters all adopt the default settings from the original paper.

The prompt template for the training set and test set is as: "[Question] Let's think step by step and output the final answer within $\backslash$boxed\{\}", [Question] represents a specific mathematical problem.

\begin{table*}[h!]
\centering
\caption{Training configurations of mathematical reasoning.}
\scalebox{.9}{
\setlength{\tabcolsep}{1mm}
\begin{tabular}{clcclcclc}
\toprule
\multirow{4}{*}{data} & train\_batch\_size    & \multicolumn{1}{c|}{128} & \multirow{4}{*}{actor} & ppo\_mini\_batch  & \multicolumn{1}{c|}{128}  & \multirow{4}{*}{rollout} & temperature             & 1.2 \\
                      & max\_prompt\_length   & \multicolumn{1}{c|}{1K}  &                        & clip\_ratio\_low  & \multicolumn{1}{c|}{0.2}  &                          & n                       & 8   \\
                      & max\_response\_length & \multicolumn{1}{c|}{4K}  &                        & clip\_ratio\_high & \multicolumn{1}{c|}{0.28} &                          & val\_kwargs.n           & 8   \\
                      & gen\_batch\_size      & \multicolumn{1}{c|}{256} &                        & kl\_loss\_coef    & \multicolumn{1}{c|}{0}    &                          & val\_kwargs.temperature & 0.6 \\ \bottomrule
\end{tabular}
}
\label{tab: math config}
\end{table*}

\subsection{Function Calling.}
\label{sec: app tool}
The baseline for Function Calling is TooRL \cite{toolrl}, which is the earliest open-source method to introduce GRPO into Function Calling. Its main contribution lies in a series of reward designs related to function calling, where the reward is no longer a binary 0 or 1 but a range of $[-3, 4]$. The validation rewards of TooRL includes formatting scores and match scores, match scores further includes tool name matching, parameter name matching, and parameter content matching. In order to adapt to DiPO, we denoted the sample with the maximum reward as the correct sample, and other samples as error samples. During BRR, we use 4 and -3 for the correct and error reward reallocation respectively. We implement DAPO and DiPO on 8 A100 GPUs, the hyperparameter settings for DAPO and DiPO are shown in Table \ref{tab: tool config}, which are consistent with TooRL except DAPO's private parameters. The training dataset and system prompt are both consistent with TooRL.

\begin{table*}[h!]
\centering
\caption{Training configurations of function calling.}
\scalebox{.9}{
\setlength{\tabcolsep}{1mm}
\begin{tabular}{clcclcclc}
\toprule
\multirow{4}{*}{data} & train\_batch\_size    & \multicolumn{1}{c|}{512} & \multirow{4}{*}{actor} & ppo\_mini\_batch  & \multicolumn{1}{c|}{128}  & \multirow{4}{*}{rollout} & temperature             & 1.2 \\
                      & max\_prompt\_length   & \multicolumn{1}{c|}{2K}  &                        & clip\_ratio\_low  & \multicolumn{1}{c|}{0.2}  &                          & n                       & 4   \\
                      & max\_response\_length & \multicolumn{1}{c|}{1K}  &                        & clip\_ratio\_high & \multicolumn{1}{c|}{0.28} &                          & val\_kwargs.n           & 4   \\
                      & gen\_batch\_size      & \multicolumn{1}{c|}{1024} &                        & kl\_loss\_coef    & \multicolumn{1}{c|}{0}    &                          & val\_kwargs.temperature & 0.6 \\ \bottomrule
\end{tabular}
}
\label{tab: tool config}
\end{table*}

\section{More Experiment Results}

\subsection{Results on Llama3.1-8B-Instruct}

To further verify the generality of DiPO, we conduct experiments using Llama3.1-8B-Instruct \cite{llama3} on GSM8K \cite{gsm8k} and MATH, as shown in Table \ref{tab: llama}. It can be seen that for Llama3.1-8B-Instruct, entropy loss and CDE are not effective methods, their results are lower than the baseline DAPO, especially for MATH; by contrast, our method shows moderate improvement. Concretely, DiPO achieves the best performance on MATH (56.75\%) and the highest overall average (73.39\%), demonstrating its consistent effectiveness across model families.

\begin{table}[b!]
\centering
\caption{Comparison of mathematical reasoning in acc/mean@8 using Llama3.1-8B-Instruct as base model. The best results is marked in \textbf{bold}.}
\scalebox{.9}{
\setlength{\tabcolsep}{3.5mm}
\begin{tabular}{l|c|c|c}
\toprule
\textbf{Method}    & {\textbf{GSM8K}} & {\textbf{MATH}} & {\textbf{AVG}}       \\ \midrule
Llama3.1-8B-Instruct & 79.80   & 47.83  & \cellcolor[RGB]{239, 243, 248}63.82 \\
DAPO     & {89.89}   & 55.75  & \cellcolor[RGB]{239, 243, 248}72.82 \\
DAPO w/ EL    & {89.66}   & 52.68  & \cellcolor[RGB]{239, 243, 248}71.17 \\
CDE     & {89.01}   & 53.45  & \cellcolor[RGB]{239, 243, 248}71.23 \\
\midrule
\textbf{DiPO (ours)}     & \textbf{90.03}   & \textbf{56.75}  & \cellcolor[RGB]{239, 243, 248}\textbf{73.39} \\ \bottomrule
\end{tabular}
}
\label{tab: llama}

\end{table}

\subsection{Results of Majority Vote}
The ACC/maj@8 metric, which determines the final answer by majority voting across 8 reasoning attempts, primarily reflects a model's consistency and reliability in mathematical problem-solving. A higher ACC/maj@8 score indicates that the model can reliably generate correct reasoning paths, not just occasionally produce the right answer.

As presented in Table \ref{tab: math maj}, the proposed DiPO method yields the highest average performance across the six mathematical benchmarks for all three base models investigated. For the Qwen3-4B-Base model, DiPO achieves an average score of 56.41\%, outperforming CDE—the second-performing method with a score of 55.51\%—by 0.90 percentage points. It attains state-of-the-art performance on AIME24, AMC and MIN, while maintaining strong competitiveness on the remaining benchmarks. For the Qwen3-8B-Base model, DiPO reaches an average score of 60.65\%, exceeding CDE (the second-best method with 59.70\%) by 0.95 percentage points and securing the top performance on four benchmarks, namely AIME24, AIME25, AMC and OLY. For the Qwen2.5-7B model, DiPO registers an average score of 49.79\%, outperforming DAPO—the second-ranked method with 48.97\%—by 0.82 percentage points and ranking first on three benchmarks: AIME24, AMC and OLY. Collectively, DiPO attains the top performance in 10 out of the 18 benchmark-model pairs, which validates its consistent and superior capability in generating reliable mathematical reasoning processes.

\begin{table*}[]
\centering
\caption{Comparison of mathematical reasoning in ACC/maj@8 on 6 mathematics benchmarks. The best and second-best results are respectively marked in \textbf{bold} and underlined. }
\scalebox{0.9}{
\setlength{\tabcolsep}{4mm}
\begin{tabular}{l|c|c|c|c|c|c|c}
\toprule
Model         & \multicolumn{1}{l|}{AIME24} & \multicolumn{1}{l|}{AIME25} & \multicolumn{1}{l|}{MATH} & \multicolumn{1}{l|}{AMC} & \multicolumn{1}{l|}{OLY} & \multicolumn{1}{l|}{MIN} & \multicolumn{1}{l}{AVG}                \\ \midrule
\multicolumn{8}{c}{\cellcolor[RGB]{239, 239, 239}Qwen3-4B-Base} \\ \midrule
Base model & 13.30 & 6.67 & 64.00 & 46.99 & 37.59 & 30.15 & \cellcolor[RGB]{239, 243, 248}33.12 \\
GRPO & \textbf{36.67} & \underline{ 26.67} & 89.60 & 68.67 & 58.10 & 50.37 & \cellcolor[RGB]{239, 243, 248}55.01 \\
DAPO & \underline{ 33.33} & \textbf{30.00} & 89.60 & \underline{ 69.88} & \underline{ 59.29} & 49.26 & \cellcolor[RGB]{239, 243, 248}55.23 \\
DAPO w/ EL & \underline{ 33.33} & \textbf{30.00} & \textbf{90.40} & 68.67 & \textbf{60.03} & \underline{ 50.40} & \cellcolor[RGB]{239, 243, 248}55.47 \\
CDE & \textbf{36.67} & \textbf{30.00} & 89.80 & \underline{ 69.88} & 57.06 & 49.63 & \cellcolor[RGB]{239, 243, 248}\underline{ 55.51} \\ \midrule
\textbf{DiPO (ours)} & \textbf{36.67} & \textbf{30.00} & \underline{ 90.20} & \textbf{72.29} & 58.59 & \textbf{50.70} & \cellcolor[RGB]{239, 243, 248}\textbf{56.41} \\ \midrule
\multicolumn{8}{c}{\cellcolor[RGB]{239, 239, 239}Qwen3-8B-Base} \\ \midrule
Base model & 16.67 & \underline{ 13.33} & 80.20 & 59.04 & 44.57 & 36.76 & \cellcolor[RGB]{239, 243, 248}41.76 \\
GRPO & 36.67 & \textbf{33.33} & 92.40 & 77.11 & \underline{ 62.11} & \textbf{54.41} & \cellcolor[RGB]{239, 243, 248}59.34 \\
DAPO & 36.67 & \textbf{33.33} & 92.40 & 77.11 & 61.07 & \underline{ 53.68} & \cellcolor[RGB]{239, 243, 248}59.04          \\
DAPO w/ EL & \underline{ 40.00} & \textbf{33.33} & 92.00 & \underline{ 78.31} & 61.37 & 52.21 & \cellcolor[RGB]{239, 243, 248}59.54 \\
CDE & \underline{ 40.00} & \textbf{33.33} & \textbf{93.20} & 75.90 & \underline{ 62.11} & \underline{ 53.68} & \cellcolor[RGB]{239, 243, 248}\underline{ 59.70} \\ \midrule
\textbf{DiPO (ours)} & \textbf{43.30} & \textbf{33.33} & \underline{ 92.60} & \textbf{79.52} & \textbf{62.90} & 52.25 & \cellcolor[RGB]{239, 243, 248}\textbf{60.65} \\ \midrule
\multicolumn{8}{c}{\cellcolor[RGB]{239, 239, 239}Qwen2.5-7B} \\ \midrule
Base model & 6.67 & 3.33 & 52.20 & 27.71 & 25.85 & 20.59 & \cellcolor[RGB]{239, 243, 248}22.73 \\
GRPO & 26.67 & \textbf{26.67} & \underline{ 84.40} & 65.06 & 46.95 & 41.18 & \cellcolor[RGB]{239, 243, 248}48.49 \\
DAPO & 30.00 & \underline{23.33} & 83.40 & \textbf{67.47} & 47.70 & \underline{ 41.91} & \cellcolor[RGB]{239, 243, 248}\underline{ 48.97} \\
DAPO w/ EL & 30.00 & 20.00 & 83.40 & \underline{ 63.86} & 48.14 & \textbf{43.01} & \cellcolor[RGB]{239, 243, 248}48.07 \\
CDE & 23.30 & 20.00 & \textbf{85.00} & \underline{ 63.86} & \underline{ 49.33} & 41.18 & \cellcolor[RGB]{239, 243, 248}47.11 \\ \midrule
\textbf{DiPO (ours)} & \textbf{33.33} & \underline{ 23.33} & \underline{ 84.00} & \textbf{67.47} & \textbf{49.44} & 41.18 & \cellcolor[RGB]{239, 243, 248}\textbf{49.79} \\ \bottomrule
\end{tabular}
}
\label{tab: math maj}
\end{table*}
\subsection{Coefficient Sensitivity Analysis}
\label{sec: app coef}
As reported in Table \ref{tab: coeff}, we analyzed the parameter sensitivity, which demonstrates the distinct robustness characteristics between the proposed DiPO and entropy loss. For the Qwen3-8B-Base model, DiPO with a coefficient of 0.10 achieves the optimal average score of 54.79\%, representing a +1.56 improvement over the DAPO baseline (53.23\%). Notably, even with a tenfold larger coefficient of 1.00, DiPO maintains a stable AVG of 53.36\%, which remains marginally above the baseline. The performance variation across this tenfold coefficient change is only 1.43 AVG points. In contrast, the entropy loss exhibits significantly higher sensitivity; a small coefficient of 0.001 yields a modest gain (AVG 53.90\%, +0.67), increasing it merely to 0.01 causes severe performance degradation, dropping the AVG to 46.00\%—a decrease of 7.90 points and 7.23 points below the baseline. This drastic collapse is consistent across all benchmarks; for instance, on the AMC dataset, performance plummets from 69.87\% to 56.33\%. A similar pattern is observed with the 4B model, where DiPO's performance remains stable between coefficients of 0.10 and 1.00 (AVG 50.75\% vs. 49.66\%), while entropy loss at 0.01 causes the AVG to fall to 42.47. The results indicate that DiPO provides a much wider and more forgiving effective coefficient range, offering reliable performance gains without the risk of catastrophic collapse, thereby presenting a more robust and practical RL strategy.

\begin{table*}[]
\centering
\caption{The impact of the coefficients for DiPO and entropy loss on the results.}
\scalebox{0.9}{
\setlength{\tabcolsep}{3mm}
\begin{tabular}{l|c|c|c|c|c|c|c|c}
\toprule
\textbf{Mathod} & \textbf{coeff} & \textbf{AIME24} & \textbf{AIME25} & \textbf{MATH} & \textbf{AMC} & \textbf{OLY} & \textbf{MIN} & \textbf{AVG} \\ \midrule
\multicolumn{9}{c}{\cellcolor[RGB]{239, 239, 239}Qwen3-4B-Base} \\ \midrule
Baseline & 0.00 & 26.25 & 23.75 & 86.43 & 61.90 & 53.88 & 44.34 & \cellcolor[RGB]{239, 243, 248}49.43 \\ \midrule
 & 0.10 & 29.17 & 24.58 & 87.00 & 64.91 & 54.09 & 44.76 & \cellcolor[RGB]{239, 243, 248}50.75 \\
\multirow{-2}{*}{DiPO} & 1.00 & 25.83 & 24.17 & 86.38 & 62.05 & 54.01 & 45.50 & \cellcolor[RGB]{239, 243, 248}49.66 \\ \midrule
 & 0.001 & 26.67 & 24.58 & 86.78 & 62.95 & 54.53 & 44.53 & \cellcolor[RGB]{239, 243, 248}50.01 \\
\multirow{-2}{*}{DAPO /w EL} & 0.01 & 18.33 & 20.00 & 81.12 & 50.00 & 45.99 & 39.38 & \cellcolor[RGB]{239, 243, 248}42.47 \\ \midrule
\multicolumn{9}{c}{\cellcolor[RGB]{239, 239, 239}Qwen3-8B-Base} \\ \midrule
Baseline & 0.00 & 30.08 & 25.83 & 89.43 & 69.12 & 56.90 & 48.02 & \cellcolor[RGB]{239, 243, 248}53.23 \\ \midrule
 & 0.10 & 35.00 & 27.50 & 89.55 & 71.23 & 57.73 & 47.75 & \cellcolor[RGB]{239, 243, 248}54.79 \\
\multirow{-2}{*}{DiPO} & 1.00 & 32.08 & 24.58 & 88.68 & 70.03 & 56.71 & 48.07        & \cellcolor[RGB]{239, 243, 248}53.36 \\ \midrule
\multicolumn{1}{c|}{} & 0.001 & 33.75 & 25.42 & 89.58 & 69.87 & 57.21 & 47.56 & \cellcolor[RGB]{239, 243, 248}53.90 \\
\multicolumn{1}{c|}{\multirow{-2}{*}{DAPO /w EL}} & 0.01 & 22.50 & 20.42 & 84.18 & 56.33 & 49.49 & 43.06 & \cellcolor[RGB]{239, 243, 248}46.00 \\ \bottomrule
\end{tabular}
}
\label{tab: coeff}
\end{table*}
\subsection{Results of Risk Prediction}
In the contemporary landscape of digital platform management, real-time public opinion risk prediction is a critical capability for safeguarding the operational integrity and brand reputation. This task involves a complex analysis of user-generated content to identify potential crises across various dimensions, such as product experience, fraud, and regulatory compliance. It can be realized through LLMs guided by a specialized prompt that mandates a ``risk-first'' priority to ensure that no negative sentiment is overlooked. As detailed in the \ref{fig: no risk prompt}, the LLM is instructed to maintain strict standards for ``no-risk'' classification, where any ambiguity or negative indicator automatically triggers a risk label. Due to confidentiality reasons, the content within the brackets ``[]'' in the prompt is replaced with placeholders.

To evaluate our approach in risk prediction, we constructed 4000 data (3000 for training and 1000 for test), and conducted a comparative analysis of DiPO and DAPO, trained on Qwen3-8B model \cite{qwen3}. As reported in Table \ref{tab: norisk}, the experimental results reveal that Qwen3-8B achieves only a 52.06\% accuracy, our {DiPO} method achieves state-of-the-art performance with the highest {Accuracy (78.37\%)}, {Recall (79.49\%)}, and {F1 Score (86.84\%)}. Although the DAPO model maintains a slightly higher precision of 95.96\%, DiPO's superior recall is particularly vital in an emergency analysis context, as it minimizes the likelihood of missing critical risks while still maintaining an exceptionally high precision of 95.69\%. This balanced performance demonstrates that DiPO is the most robust and reliable framework for automated public opinion monitoring, providing an effective tool for identifying and mitigating potential threats in a complex information ecosystem.

\begin{figure*}[h]
  \centering
   \includegraphics[width=1\linewidth]{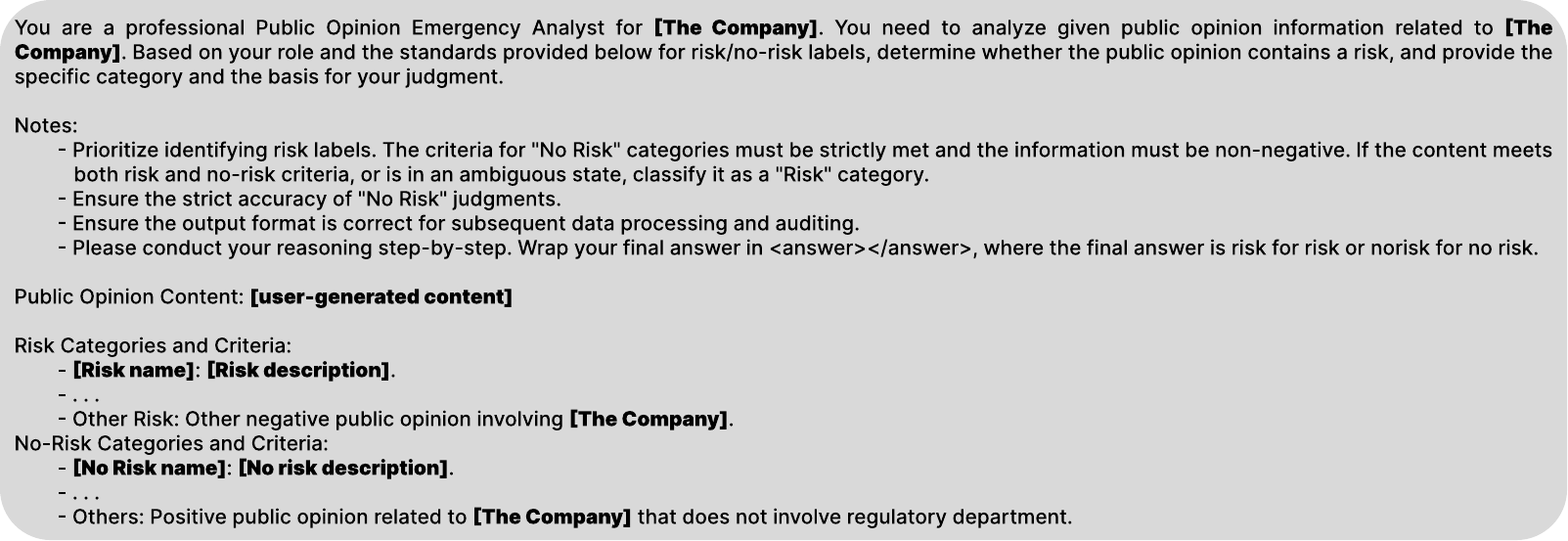}
    \caption{Prompt of risk prediction.}
    \label{fig: no risk prompt}
\end{figure*}

\begin{table*}[h]
\centering
\caption{The results of risk prediction, including accuracy, precision, recall, and F1 score, and all results are the mean of 8 independent inferences.}
\scalebox{1}{
\setlength{\tabcolsep}{3mm}
\begin{tabular}{l|c|c|c|c}
\toprule
\textbf{Method}   & \textbf{ACC/mean@8}     & \textbf{Precision/mean@8} & \textbf{Recall/mean@8}  & \textbf{F1/mean@8}      \\ \midrule
Qwen3-8B & 52.06          & 87.18            & 54.64          & 67.18          \\
DAPO     & 76.94          & \textbf{95.96}   & 77.58          & 85.80          \\
DiPO     & \textbf{78.37} & 95.69            & \textbf{79.49} & \textbf{86.84} \\ \bottomrule
\end{tabular}
}
\label{tab: norisk}
\end{table*}

\subsection{Visualization of PPL Distribution}
\label{sec: app ppl}
To further analyze the exploration and exploitation trends during the RL training process, we conducted a visualization analysis of the PPL distribution during the training of DAPO and DiPO. As shown in Figure \ref{fig: ppl distribution step}, during the initial training stage, the PPL distributions of DAPO and DiPO are relatively consistent, and it is difficult to distinguish the PPL distributions of correct and error samples, which is the \textbf{motivation for introducing advantage judgment} in PSD. In the later stages of training, the PPL distributions of DAPO and DiPO gradually converge. The difference is that for DAPO, the overall PPL will converge to a lower range (whether it is error samples or correct samples), while PPL distribution of DiPO is more discriminative, with error samples in a higher PPL range and correct samples in a lower PPL range, reflecting a reasonable tendency for exploration and exploitation.

\subsection{Case Analysis}
Figure \ref{fig: correct DiPO} \ref{fig: correct DAPO} \ref{fig: error DiPO} \ref{fig: error DAPO} show some cases of extreme groups for DAPO and DiPO trained to the 500-th step on DAPO-17K, where darker colors indicate higher entropy. It can be seen that under correct answers, the overall entropy of DiPO cases is smaller, and under incorrect answers, the density of high-entropy tokens in DiPO is greater, which also reflects the effective exploitation and exploration balance of DiPO.

\begin{figure*}[]
  \centering
   \includegraphics[width=1\linewidth]{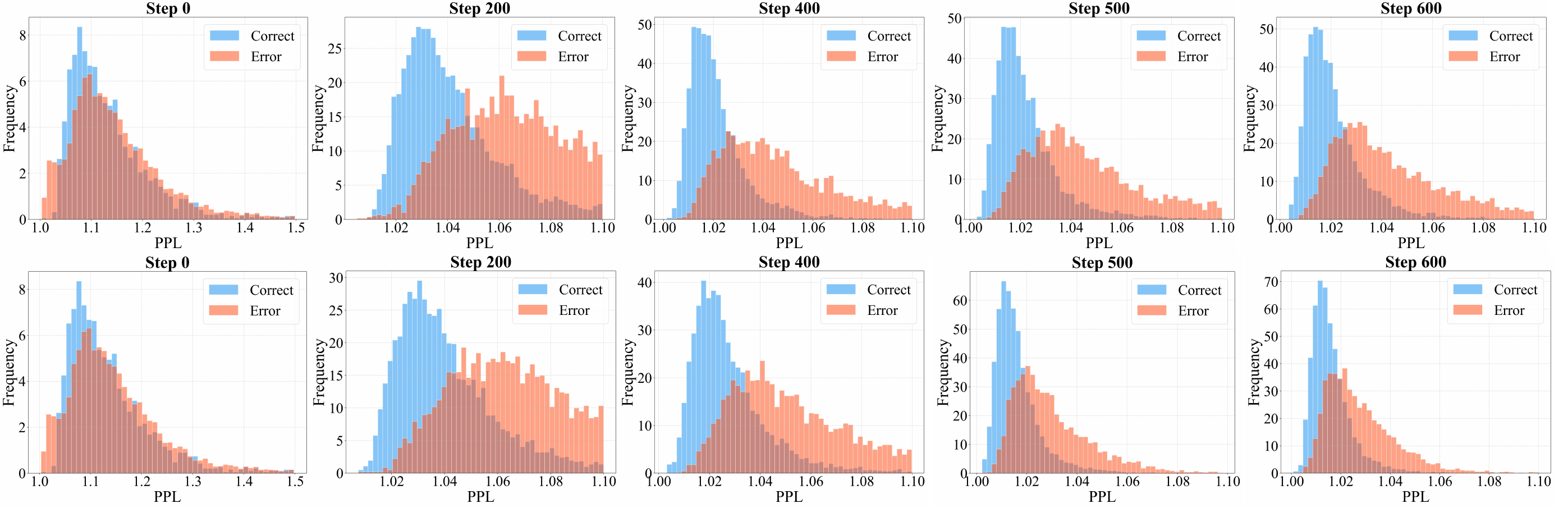}
    
    \caption{PPL distribution of correct and error samples for Qwen3-8B-Base trained on DAPO-17K Dataset via DiPO (TOP) and DAPO (Bottom).}
    \label{fig: ppl distribution step}
\end{figure*}

\begin{figure*}[]
  \centering
   \includegraphics[width=1\linewidth]{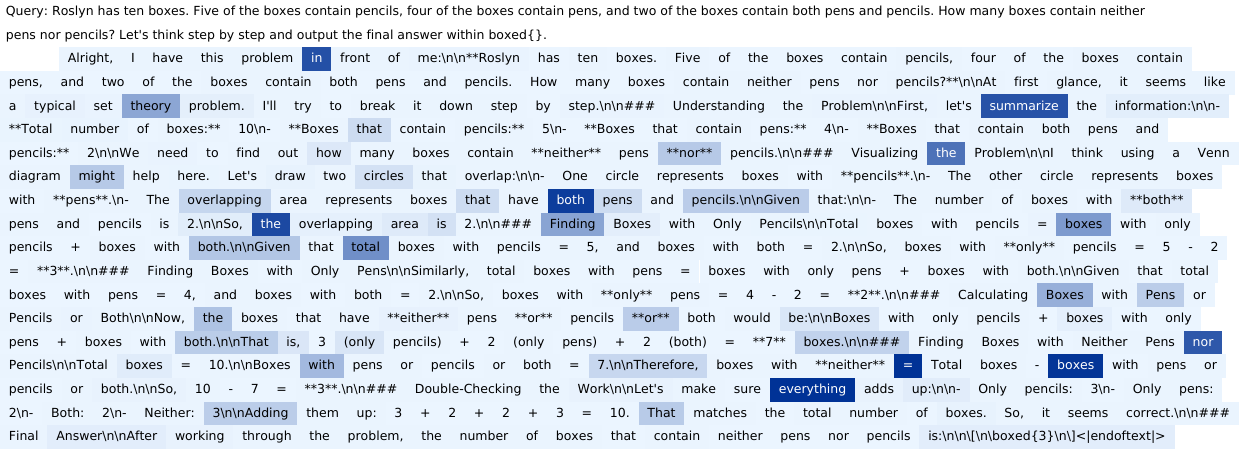}
    \caption{Correct case of DAPO.}
    \label{fig: correct DAPO}
\end{figure*}

\begin{figure*}[]
  \centering
   \includegraphics[width=1\linewidth]{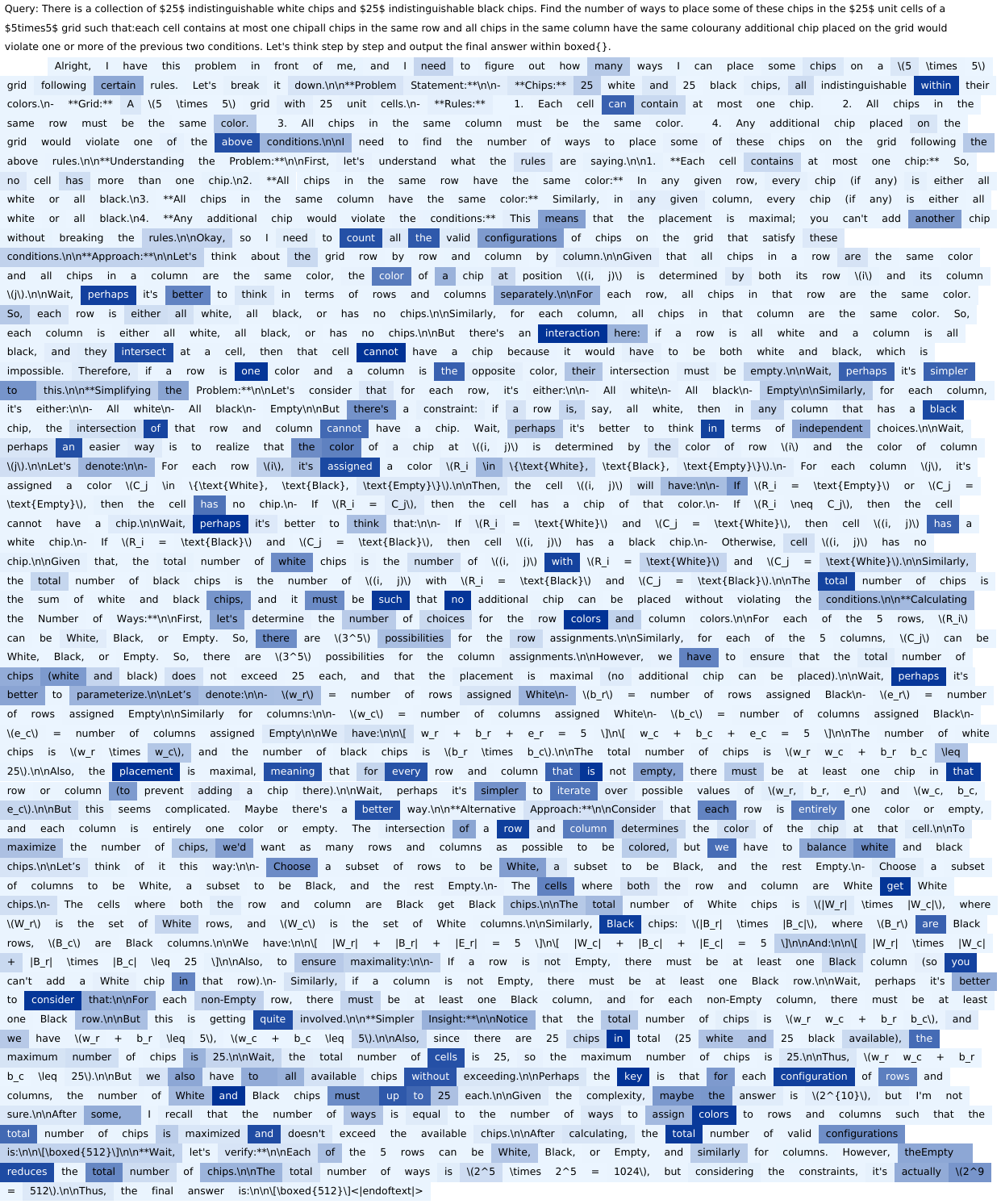}
    \caption{Error case of DAPO.}
    \label{fig: error DAPO}
\end{figure*}

\begin{figure*}[t!]
   \includegraphics[width=1\linewidth]{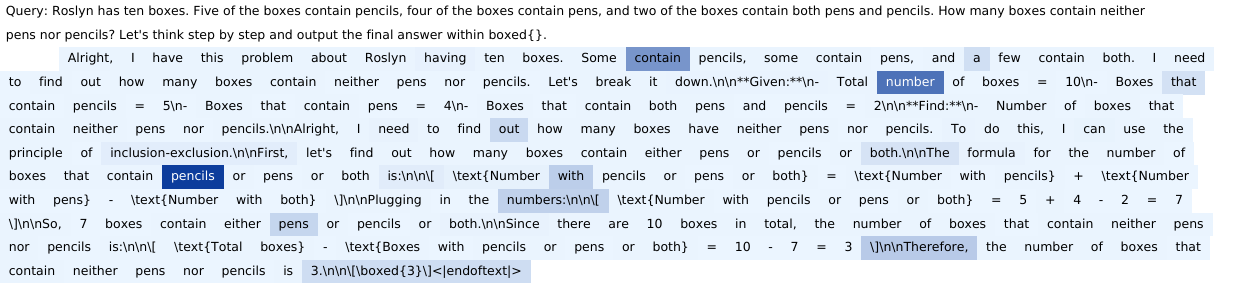}
    \caption{Correct case of DiPO.}
    \label{fig: correct DiPO}
\end{figure*}

\begin{figure*}[]
  \centering
   \includegraphics[width=1\linewidth]{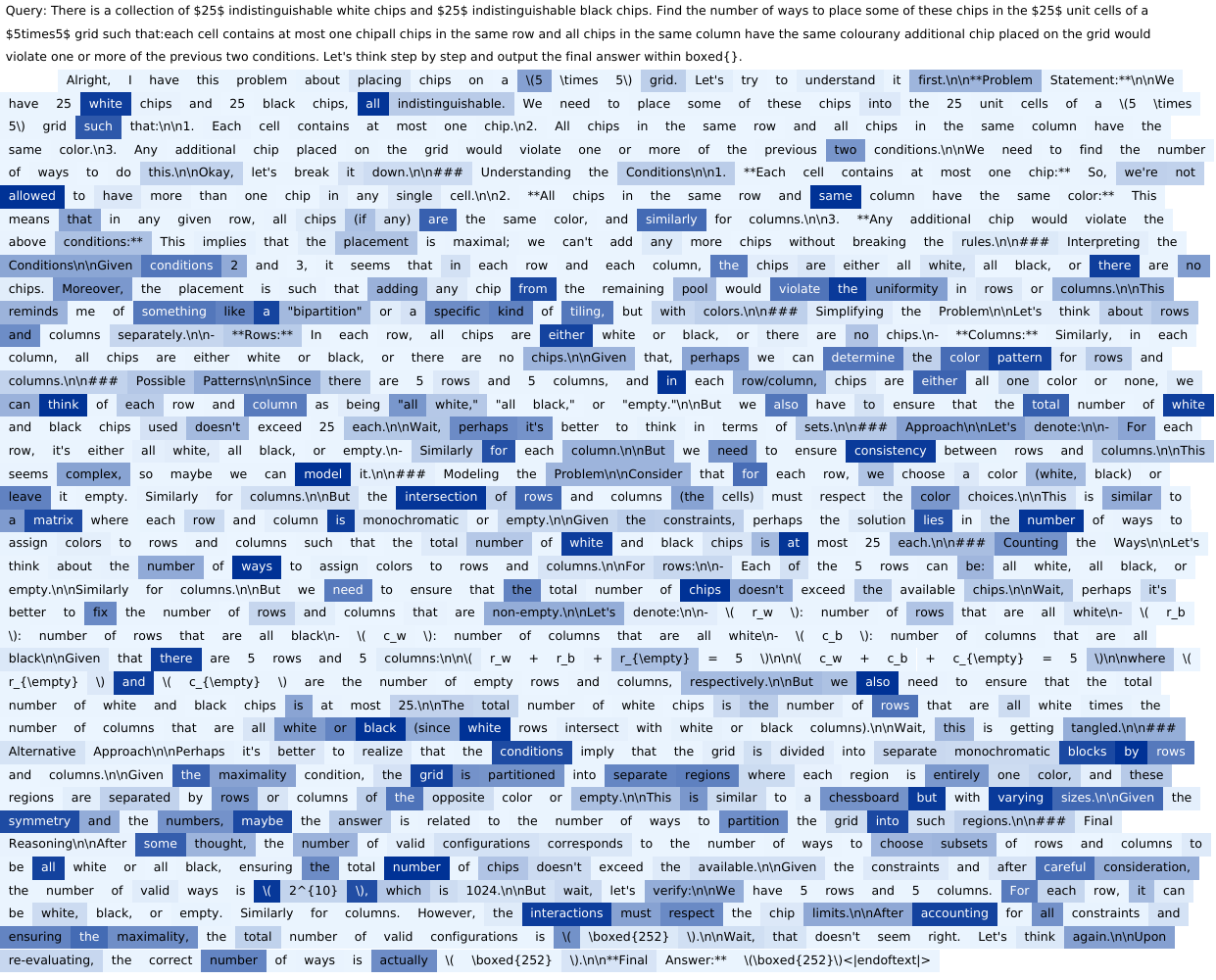}
    \caption{Error case of DiPO.}
    \label{fig: error DiPO}
\end{figure*}


\newpage

\end{document}